\documentclass[11pt]{article}
\pdfoutput=1
\usepackage{longtable}
\usepackage{wrapfig}
\usepackage[utf8]{inputenc} 
\usepackage[T1]{fontenc}    
\usepackage[colorlinks,
            linkcolor=red,
            anchorcolor=blue,
            citecolor=blue
            ]{hyperref}
\usepackage{url}            
\usepackage{booktabs}       
\usepackage{amsfonts}       
\usepackage{nicefrac}       
\usepackage{xcolor,colortbl}         
\usepackage{wrapfig}
\usepackage{caption}
\usepackage{enumitem}
\usepackage{tabularx}
\usepackage{mylatexstyle}
\usepackage[compact]{titlesec}
\usepackage{multirow}
\usepackage{setspace}
\usepackage{fullpage}

\usepackage[protrusion=true,expansion=true]{microtype}
\usepackage{pbox}
\usepackage{setspace}
\usepackage{tabularx}
\usepackage{float}
\usepackage{wrapfig,lipsum}
\usepackage{enumitem}
\usepackage{graphicx}
\usepackage{subfigure}
\usepackage{pgfplots}
\usetikzlibrary{arrows,shapes,snakes,automata,backgrounds,petri}
\usepackage{booktabs} 

\allowdisplaybreaks
\usepackage{colortbl}
\definecolor{Gray}{rgb}{0.8, 0.9, 1}
\definecolor{LightGray}{gray}{0.9}
\definecolor{codepurple}{rgb}{0.58,0,0.82}
\definecolor{VioletRed}{rgb}{0.915686,0.025490,0.364706}

\def \method {MARS-M}

\usepackage{enumitem}
\usepackage{colortbl}
\definecolor{LightCyan}{rgb}{0.8, 0.9, 1}
\definecolor{Gray}{gray}{0.86}

\ifdefined\final
\usepackage[disable]{todonotes}
\else
\usepackage[textsize=tiny]{todonotes}
\fi
\title{\huge MARS-M: When Variance Reduction Meets Matrices}

\author
{
    Yifeng Liu\thanks{Equal contribution} \thanks{Department of Computer Science, University of California, Los Angeles, CA, USA; email: liuyifeng@cs.ucla.edu} 
	~~~
    Angela Yuan\footnotemark[1] \thanks{Department of Computer Science, University of California, Los Angeles, CA, USA; email: hzyuan@cs.ucla.edu} 
	~~~
	Quanquan Gu\footnotemark[2] \thanks{Corresponding Author, Department of Computer Science, University of California, Los Angeles, CA, USA; email: qgu@cs.ucla.edu}
}

\date{}

\begin{document}

\maketitle

\begin{abstract} Matrix-based preconditioned optimizers, such as Muon, have recently been shown to be more efficient than scalar-based optimizers for training large-scale neural networks, including large language models (LLMs). Recent benchmark studies of LLM pretraining optimizers have demonstrated that variance-reduction techniques such as MARS can substantially speed up training compared with standard optimizers that do not employ variance reduction.
In this paper, we introduce MARS-M, a new optimizer that integrates MARS-style variance reduction with Muon. Under standard regularity conditions, we prove that MARS-M converges to a first-order stationary point at a rate of $\tilde{\mathcal{O}}(T^{-1/3})$, improving upon the $\tilde{\mathcal{O}}(T^{-1/4})$ rate attained by Muon. Empirical results on language modeling and computer vision tasks demonstrate that MARS-M consistently yields lower losses and improved performance across various downstream benchmarks. The implementation of \text{\method} is available at \url{https://github.com/AGI-Arena/MARS/tree/main/MARS_M}.
\end{abstract}


\section{Introduction}
The development of preconditioned gradient methods, such as AdamW~\citep{Adamw}, AdaFactor~\citep{shazeer2018adafactor}, and Lion~\citep{Lion}, has played an important role in the advancement of large-scale deep learning. Many prominent large language models (LLMs), including ChatGPT~\citep{chatgpt}, LLaMA-3~\citep{llama3}, and DeepSeek-R1~\citep{guo2025deepseek}, are trained with adaptive gradient methods such as Adam~\citep{adam} and AdamW. Recently, matrix-based preconditioned optimization methods, such as Shampoo~\citep{gupta2018shampoo}, SOAP~\citep{SOAP}, and Muon~\citep{muon,liu2025muon}, have been introduced to accelerate the training of large models such as Kimi K2~\citep{team2025kimi} and GLM-4.5~\citep{zeng2025glm}. Unlike vector-based methods, matrix-based approaches operate directly on parameter matrices without flattening them, thereby preserving their inherent 2D structure and the underlying optimization geometry.

On the other hand, stochastic optimization methods are often hindered by high variance in stochastic gradients during training, which can slow convergence and degrade stability. To address this issue, numerous variance-reduction techniques have been proposed, including SAG~\citep{roux2012stochastic}, SVRG~\citep{johnson2013accelerating}, SARAH~\citep{SARAH1,SARAH2}, SPIDER~\citep{fang2018spider}, SNVRG~\citep{zhou2020stochastic}, and STORM~\citep{STORM}, to name a few. However, these methods have seen limited success in training large-scale deep neural networks, largely due to incompatibilities with modern deep-learning practice and neural network architectures~\citep{defazio2019ineffectiveness}. To make variance reduction work for training large-scale deep neural networks and LLMs, MARS~\citep{yuan2024mars} was recently proposed. It introduces a scaling parameter into the STORM optimizer~\citep{STORM}, effectively resolving the incompatibility between variance-reduction techniques and preconditioned optimization methods, which have demonstrated superior performance in both language modeling and computer vision tasks. A similar scaling idea has also been proposed for SVRG in \citet{yin2023coefficient}.

Recent benchmarks of optimizers for LLM pretraining~\citep{wen2025fantastic,semenov2025benchmarking} have empirically shown that matrix-based optimizers (e.g., Shampoo, SOAP, and Muon) outperform scalar-based optimizers (e.g., AdamW and Lion), and that variance-reduction approaches (e.g., MARS and NadamW~\citep{dozat2016incorporating}) can yield additional, discernible speedups. Therefore, a natural question arises:
\begin{center}
\emph{Can we achieve the best of both worlds by combining matrix-based optimizers with variance reduction such as MARS?}
\end{center}
It is worth noting that the original MARS work~\citep{yuan2024mars} proposed MARS-Shampoo, which attempts to combine a matrix-based optimizer with variance reduction. However, MARS-Shampoo performs worse than vector-based variants such as MARS-AdamW, leaving open the question of whether variance reduction can be effectively integrated with matrix-based optimizers. In this paper, we give an affirmative answer by introducing MARS-M, a matrix-based MARS optimizer that integrates variance reduction with the Moonlight version~\citep{liu2025muon} of Muon. Additionally, we propose an approximate version of MARS-M as a practical variant, which also serves as a bridge between variance-reduction techniques and traditional preconditioned optimizers.

In summary, our contributions are highlighted as follows:
\begin{itemize}
    \item We propose MARS-M, an instantiation of MARS tailored to the Moonlight optimizer. In addition, we develop an approximate version of MARS-M to further accelerate training. We show the connections and differences between MARS-M and MARS-Shampoo. We also show that the approximate version of MARS-M can be seen as a variant of Moonlight with adjusted momentum parameters.
    \item We provide a convergence analysis of MARS-M and show that it attains a convergence rate of $\tilde{\mathcal{O}}(T^{-1/3})$, which improves upon the previously established $\mathcal{O}(T^{-1/4})$ rate achieved by Muon~\citep{li2025note,shen2025convergence,pethick2025training}.
    \item Empirically, we evaluate the performance of MARS-M as well as the Muon (Moonlight) optimizer for training GPT-2~\citep{radford2019language} series models on OpenWebText and FineWeb-Edu 100B datasets. The experimental results show that MARS-M delivers consistent improvements in training and validation losses. Moreover, on downstream tasks, MARS-M achieves better performance on benchmarks such as Hellaswag~\citep{hellaswag} and SciQ~\citep{SciQ} than baseline optimizers. Additionally, MARS-M achieves higher test accuracy than AdamW and Muon on computer vision tasks.
\end{itemize}

\noindent\textbf{Notation} We use bold capital letters $\Xb,\Yb,\cdots$ to denote matrices. For a matrix $\Wb\in \mathbb{R}^{m\times n}$, we use $\|\Wb\|_F=\sqrt{\sum_{i=1}^{m}\sum_{j=1}^n \Wb_{ij}^2}$ to denote its Frobenius norm. In this paper, we use $\mathbf{X}_t\in\mathbb{R}^{m\times n}$ to represent the parameters of the language model at training step $t$, where the training data for each step is a sequence of independent random variables $\xi_1, \dots, \xi_T \in \mathbf{\Xi}$. For a differentiable objective function $f$, we assume that the expected value of $f(\mathbf{X}, \xi_t)$ given $\mathbf{X}$ is $F(\mathbf{X})$ for all $\mathbf{X}$ and $t$. In this paper, we may explicitly denote the input data, since variance-reduction algorithms may utilize both the training data from the current step ($\xi_t$) and the previous step ($\xi_{t-1}$) to compute different gradients for the same parameter. MARS-M is directly inspired by the success of matrix-based optimizers Muon and Moonlight.

\section{Related Work}

\noindent\textbf{Preconditioned gradient methods} Preconditioned gradient methods, inspired by Newton’s method, are a powerful class of optimization algorithms widely used in training deep neural networks. To reduce the computational and memory costs of computing and storing the full Hessian matrix, many approximate preconditioners have been proposed, including diagonal approximations~\citep{adagrad,Adamw,liusophia,yao2021adahessian} and sketched approximations~\citep{erdogdu2015convergence,gonen2015faster}. Despite these advances, only a few matrix-based optimization methods have been practical at large scale. A notable breakthrough was the Shampoo optimizer~\citep{gupta2018shampoo}, which demonstrated the practical feasibility of incorporating second-order information in large-scale training.
To reduce the number of hyperparameters and computational overhead of Shampoo,~\citet{SOAP} proposed SOAP, which implements AdamW in the eigenbasis of Shampoo's preconditioners, significantly improving training efficiency in terms of both iteration count and wall-clock time. Recently, Muon~\citep{muon} was proposed by leveraging a Newton--Schulz approximation of the singular value decomposition (SVD) of the gradient momentum, while Moonlight~\citep{liu2025muon} extends Muon by introducing weight decay and coupling step sizes across components optimized by AdamW and Muon, thereby improving hyperparameter search efficiency. Building on this line of research, PolarGrad~\citep{lau2025polargrad} introduces a preconditioning framework based on the polar decomposition of the gradient matrix, generalizing the principles of Muon. Scion~\citep{pethick2025training} further unifies Muon and related methods under a linear minimization oracle (LMO) framework, and Gluon~\citep{riabinin2025gluon} extends both Muon and Scion by incorporating momentum-based generalizations. Today, matrix-based optimizers are becoming increasingly popular for training industrial large language models, including Kimi K2~\citep{team2025kimi} and GLM-4.5~\citep{zeng2025glm}.

To understand the convergence of Muon,~\citet{li2025note},~\citet{shen2025convergence}, and~\citet{pethick2025training} showed that Muon has a convergence rate of $\mathcal{O}(T^{-1/4})$. \citet{an2025asgo} and~\citet{shen2025convergence} also analyzed the convergence of Muon without momentum. \citet{kovalev2025understanding} investigated the convergence rate of Muon under additional assumptions such as spectral norm-based Lipschitz smoothness and second-order smoothness, while~\citet{li2025note},~\citet{sfyraki2025lions}, and~\citet{Sato2025ConvergenceBA} analyzed the convergence rate under large batch sizes. Our analysis of MARS-M pushes the theoretical frontier of Muon-based optimization algorithms.

\noindent\textbf{Variance Reduction Methods.}  The earliest efforts to accelerate stochastic gradient descent (SGD) through variance reduction include SAG~\citep{roux2012stochastic} and SDCA~\citep{shalev2013stochastic}. These were soon followed by simpler yet equally effective algorithms such as SVRG~\citep{johnson2013accelerating} and SAGA~\citep{defazio2014saga}, which achieved the same improved convergence guarantees. Building on this line of work, SARAH~\citep{SARAH1} introduced a biased recursive gradient estimator that reduces memory requirements while retaining optimal complexity bounds for convex optimization. Additionally, variance reduction has also been studied in conjunction with preconditioning in the convex setting~\citep{frangella2024promise,derezinski2023stochastic}. In the non-convex regime, algorithms like SVRG~\citep{allen2016improved,reddi2016stochastic} and SARAH~\citep{SARAH2} paved the way for methods such as SPIDER~\citep{fang2018spider}, which integrates normalized gradient descent~\citep{nesterov2013introductory,hazan2015beyond} with stochastic path-integral differential estimator, and SNVRG~\citep{zhou2020stochastic}, which leverages multiple reference points to enhance variance reduction. SpiderBoost~\citep{wang2019spiderboost} further improved SPIDER by permitting large constant step sizes without compromising near-optimal oracle complexity. Subsequently, STORM~\citep{STORM} streamlined SPIDER and SNVRG via stochastic recursive momentum, and was later extended into a parameter-free variant, STORM+~\citep{levy2021storm+}. Recent work has also explored combining variance reduction with adaptive gradient methods. For instance, Adam$^+$~\citep{liu2020adam} reduces variance in Adam by estimating gradients only at extrapolated points, while SuperAdam~\citep{huang2021super} and VRAdam~\citep{li2024smoothness} integrate variance reduction into AdamW to accelerate convergence. AdaSPIDER~\citep{kavis2022adaptive} extends SPIDER with adaptive step sizes. However, these variance-reduced adaptive optimization algorithms have been evaluated only on relatively simple computer vision and natural language modeling benchmarks with modest model sizes. To our knowledge, MARS~\citep{yuan2024mars} is the first method to achieve stellar performance on large language models. This work pointed out that previous algorithms introduced excessive gradient correction and it incorporates scaled gradient correction into adaptive gradient methods, achieving better performances in language modeling and computer vision tasks. MARS-M is built upon MARS, while utilizing matrix-based optimizer Muon/Moonlight.

\section{Preliminaries}
In this section, we present the problem setting and relevant preliminaries, including a review of the Muon and MARS optimizers.

In this paper, we consider minimizing an objective function $F(\cdot): \RR^{m\times n} \rightarrow \RR$ as follows:
\begin{align}
\min_{\Xb} F(\Xb) = \EE_{\bxi\sim \cD}[f(\Xb,\bxi)],
\end{align}
where $f(\Xb,\bxi)$ is a possibly nonconvex loss function, $\Xb \in \RR^{m\times n}$ is a matrix-variate optimization variable, and $\bxi$ is a random vector (e.g., a training data point) drawn from an unknown data distribution $\cD$.
We assume access to a first-order oracle, which returns an unbiased estimator of the gradient $\EE[\nabla f(\Xb, \bxi)] = \nabla F(\Xb)$. Throughout the paper, without loss of generality, we assume $m\geq n$.

\subsection{Muon}
Muon~\citep{muon} was proposed to exploit the 2D geometric information of model parameter matrices during optimization. Given the momentum $\beta \in (0,1)$ and learning rate $\eta_t>0$, the update rule for Muon is as follows:
\begin{align}
    \Mb_t &= \beta \Mb_{t-1} + \nabla f(\Xb_{t},\bxi_{t}),\label{eq:muon_m}
        \\
    \mathbf{O}_t &= \text{NewtonSchulz}\left(\Mb_t\right),\label{eq:muon_NS}
        \\
    \Xb_{t+1} &= \Xb_t-\eta_t\mathbf{O}_t.\label{eq:muon_final}
\end{align}
Here Newton--Schulz iteration~\citep{bernstein2024old} is used to approximate $\Ub_t \Vb_t$, where $\Ub_t \bSigma_t \Vb_t=\Mb_t$ is the singular value decomposition (SVD) of the momentum matrix.
In practice, NewtonSchulz iteration is usually applied to $\tilde{\Mb}_t=\beta \Mb_t + \nabla f(\Xb_{t},\bxi_{t})$ rather than $\Mb_t$\footnote{In~\citet{muon}'s newest implementation, both $\nabla f(\Xb_{t},\bxi_{t})$ in $\tilde{\Mb_t}$ and in~\eqref{eq:muon_m} are substituted with $(1-\beta)\nabla f(\Xb_t,\bxi_t)$. However, in this paper, we just focus on the original algorithm.}.

It is worth noting that Muon is designed primarily for optimizing matrix-like parameters, while AdamW is often applied to vector-like parameters, including embeddings, the language model head, and RMSNorm. However, empirical evidence~\citep{yuan2024mars,semenov2025benchmarking} has shown that the original version of Muon does not perform as well in practice. \citet{liu2025muon} attributed the inferior performance to mismatched update magnitudes across the two types of parameters (i.e., matrix-like and vector-like parameters). In particular, the RMS norm of AdamW updates typically ranges from approximately 0.2 to 0.4 during LLM training. Based on this observation, they proposed a variant of Muon (referred to as ``Moonlight'' in this paper):
\begin{align}
    \Ub_t &= \beta \Ub_{t-1} + \nabla f(\Xb_{t}, \bxi_{t}),\label{eq:moonlight_u}
        \\
    \Mb_t &= \beta \Ub_t + \nabla f(\Xb_{t},\bxi_{t}),\label{eq:moonlight_m}
        \\
    \mathbf{O}_t &= \text{NewtonSchulz}\left(\Mb_t\right),\label{eq:moonlight_NS}
        \\
    \Xb_{t+1} &= \Xb_t-\eta_t(0.2\cdot\mathbf{O}_t\cdot\sqrt{\max(m, n)}+\lambda\Xb_t),\label{eq:moonlight_final}
\end{align}
where $\Xb_t\in\mathbb{R}^{m\times n}$ is the matrix-like parameter to be optimized, and $\lambda>0$ is the decoupled weight decay parameter. Such a corrected version demonstrates great performance in a lot of benchmarks~\citep{semenov2025benchmarking,wen2025fantastic} and successfully applied in training industrial large language models~\citep{team2025kimi,zeng2025glm}.
\subsection{MARS}
To accelerate the convergence of SGD, a lot of variance reduction techniques have been proposed trying to reduce the variance in update and achieve stabler and faster training. 



A typical form of variance reduction is exemplified by STORM~\citep{STORM}, which introduces a gradient-correction term in the momentum update rule:
\begin{align}
    \mb_t &=
    \beta \mb_{t-1} +
    (1 - \beta) 
    \Big[
        \nabla f(\xb_t, \bxi_t)\nonumber\\&\qquad+
        \underbrace{\frac{\beta}{1 - \beta}
        \big(
        \nabla f(\xb_t, \bxi_t)-
        \nabla f(\xb_{t-1}, \bxi_t)
        \big)}_{\text{gradient correction}}
    \Big].\nonumber
\end{align}
where $\beta>0$ is the momentum parameter. Here the noise brought by the randomness of data can be canceled out by the stochastic gradient difference term $\frac{\beta}{1-\beta}(\nabla f(\xb_t, \bxi_t) -\nabla f(\xb_{t-1}, \bxi_t))$. Based on this term, the estimation of the true gradient for the parameter in the last step $\nabla F(\xb_{t-1})$ can be transferred to the true gradient for the parameter in the current step $\nabla F(\xb_{t})$.

Based on STORM, MARS \citep{yuan2024mars} (see Algorithm~\ref{alg:mars}) introduces a scaling parameter for the gradient-correction term, yielding the corrected gradient:
\begin{align}
    \bc_t =
        \nabla f(\xb_t, \bxi_t)+
        {\color{red}\gamma_t}\underbrace{\frac{\beta}{1 - \beta}
        \big(
        \nabla f(\xb_t, \bxi_t) -
        \nabla f(\xb_{t-1}, \bxi_t)
        \big)}_{\text{scaled gradient correction}},
\end{align}
where $\gamma_t>0$ is the scaling coefficient. Moreover, it also applies gradient clipping to the corrected gradient to improve training stability:
\begin{align}
        \tilde{\mathbf{c}}_t = \text{Clip}(\mathbf{c}_t,1) =  \begin{cases}
\frac{\mathbf{c}_t}{\|\mathbf{c}_t\|_2} & \text{if } \|\mathbf{c}_t\|_2 > 1,\\
\mathbf{c}_t & \text{otherwise}.
\end{cases}\label{eq:grad_clip}
\end{align}
For the ease of exposure, the MARS algorithm is summarized in Algorithm~\ref{alg:mars}, where line 7 is the general formula for adaptive gradient methods such as AdamW~\citep{Adamw}, Lion~\citep{Lion} and Shampoo~\citep{gupta2018shampoo}, with an approximated Hessian matrix $\Hb_t$. As will be seen later, our work can be viewed as a tailored version of MARS for Muon/Moonlight.

\begin{algorithm}[t!]
\caption{MARS}
\begin{algorithmic}[1]
\label{alg:mars}
    \STATE \textbf{input:} $\xb_0, \beta, \{\gamma_t\}, \{\eta_t\}$
    \STATE Set $\mb_0\leftarrow \mathbf{0}$ and $\xb_1\leftarrow\xb_0$
    \FOR {$t=1,$ \textbf{to} $ T$}
        \STATE Sample $\bxi_t$ and let $\mathbf{c}_t = \nabla f(\xb_t, \bxi_t)+\gamma_t \frac{\beta}{1-\beta} \big(\nabla f(\xb_t, \bxi_t)-\nabla f(\xb_{t-1}, \bxi_t)\big)$
        \STATE if $\|\mathbf{c}_t\|_2 > 1$, then $\Tilde{\mathbf{c}}_t=\frac{\mathbf{c}_t}{\|\mathbf{c}_t\|_2}$ else $\tilde{\mathbf{c}}_t = \mathbf{c}_t$
        \STATE $\mb_t = \beta \mb_{t-1} + (1-\beta)\Tilde{\mathbf{c}}_t$
        \STATE $\xb_{t+1} =\arg\min_{\xb}\left\{\eta_t \left\langle \mb_t, \xb\right\rangle+\frac12 \|\xb - \xb_t\|_{\Hb_t}^2\right\}$\label{mars-update}
    \ENDFOR
\end{algorithmic}
\end{algorithm}

\section{Method}
In this section, we propose the MARS-M optimizer and establish its convergence guarantees.

\subsection{MARS-M}

 We present MARS-M, which is an instantiation of MARS by substituting the update rule~\eqref{eq:muon_final} with~\eqref{eq:moonlight_final}, and the resulted algorithm is shown in Algorithm~\ref{alg:MARS_moonlight}. MARS-M is an instantiation of MARS tailored to the Moonlight optimizer.

It is worth noting that MARS-M is related to MARS-Shampoo in~\citep{yuan2024mars}, which is another instantiation of MARS based on a simplified version of Shampoo. Specifically, the update rule in MARS-Shampoo is instantiated as:
\begin{align}
    \mathbf{U}_t, \mathbf{\Sigma}_t, \mathbf{V}_t &= \text{SVD}(\Mb_t),\nonumber
    \\
    \Xb_{t+1}&=\Xb_t-\eta_t\mathbf{U}_t\mathbf{V}_t^\top.\label{eq:SVD}
\end{align}
In practice, the SVD computation can be substituted by approximate numerical methods such as Newton iteration~\citep{lakic1998computation, higham2008functions,anil2020scalable,bernstein2024old} and Newton-Schulz iteration~\citep{schulz1933iterative,higham2008functions}. The main differences between MARS-M and MARS-Shampoo are that MARS-M scales $\Ob_t \approx \Ub_t\Vb_t^\top$ by $0.2\cdot \sqrt{\max(m,n)}$ and incorporates weight decay $\lambda \Xb_t$, both of which are essential for achieving superior performance according to \citet{liu2025muon}.

\begin{algorithm}[htb!]
\caption{MARS-M}
\label{alg:MARS_moonlight}
\begin{algorithmic}[1]
    \STATE \textbf{input:} $\Xb_0\in\mathbb{R}^{m\times n}, \lambda, \beta, \{\gamma_t\}, \{\eta_t\}$
    \STATE Set $\Mb_0\leftarrow \mathbf{0}$ and $\Xb_1\leftarrow\Xb_0$
    \FOR {$t=1,$ \textbf{to} $ T$}
        \STATE sample $\bxi_t$ and let $\mathbf{C}_t = \nabla f(\Xb_t, \bxi_t)+\gamma_t(\frac{\beta}{1-\beta})\big(\nabla f(\Xb_t, \bxi_t)-\nabla f(\Xb_{t-1}, \bxi_t)\big)$~\label{eq:mars_moonlight_c_t}
        \STATE $\Mb_t = \beta \Mb_{t-1} + (1-\beta)\text{Clip}(\mathbf{C}_t, 1)$~\label{eq:mars_moonlight_m_t}
        \STATE $\Ob_t = \text{NewtonSchulz}(\Mb_t)$
        \STATE $\Xb_{t+1} = \Xb_t - \eta_t(0.2\cdot\Ob_t\cdot\sqrt{\max(m,n)} +  \lambda \Xb_t)$
    \ENDFOR
\end{algorithmic}
\end{algorithm}

\subsection{Approximated MARS-M}
Since computing the stochastic gradients twice is significantly more expensive, in practice, an approximate version of MARS is usually utilized where $\nabla f(\Xb_{t-1}, \bxi_{t})$ is replaced by $\nabla f(\Xb_{t-1}, \bxi_{t-1})$. If clipping on $\mathbf{C}_t$ is ignored, lines 4 to 5 in Algorithm~\ref{alg:MARS_moonlight} can be approximated by:
\begin{align}
    \mathbf{C}_t &= \nabla f(\Xb_t, \bxi_t)\nonumber\\
    &\qquad+\gamma_t\bigg(\frac{\beta}{1-\beta}\bigg)\big(\nabla f(\Xb_t, \bxi_t)-\nabla f(\Xb_{t-1}, \bxi_{t-1})\big),~\label{eq:approx_c_t}\\
    \Mb_t &= \beta \Mb_{t-1} + (1-\beta)\mathbf{C}_t.~\label{eq:approx_m_t}
\end{align}
With some calculation, it can be shown that the approximate version of MARS-M without clipping is equivalent to the following Muon/Moonlight-style update rule:
\begin{align}
    \Ub_t &= \beta \Ub_{t-1} + {\color{red}\frac{(1-\gamma_t)(1-\beta)}{\beta}}\nabla f(\Xb_{t}, \bxi_{t}),\label{eq:mars_moonlight_u}
        \\
    \Mb_t &= \beta \Ub_t + {\color{red}\gamma_t}\nabla f(\Xb_{t},\bxi_{t}),\label{eq:mars_moonlight_m}
        \\
    \mathbf{O}_t &= \text{NewtonSchulz}\left(\Mb_t\right),\nonumber
        \\
    \Xb_{t+1} &= \Xb_t-\eta_t(0.2\cdot\mathbf{O}_t\cdot\sqrt{\max(m, n)}+\lambda\Xb_t).\nonumber
\end{align}
By comparing \eqref{eq:mars_moonlight_u} and \eqref{eq:mars_moonlight_m} with \eqref{eq:moonlight_u} and \eqref{eq:moonlight_m}, we can see that approximate MARS-M can be viewed as a variant of Moonlight with adjusted momentum parameters $\color{red}(1-\gamma_t)(1-\beta)/\beta$ and $\color{red}\gamma_t$ (in contrast to the coefficients equal to 1 in Moonlight). 

From another perspective, we note that \eqref{eq:moonlight_u} and \eqref{eq:moonlight_m} in the original Moonlight algorithm can be equivalently expressed as:
\begin{align}
    \mathbf{C}_t &= {\color{blue}\frac{1}{1-\beta}}\nabla f(\Xb_t, \bxi_t)\nonumber\\
    &\qquad+{\color{blue}1}\cdot\bigg(\frac{\beta}{1-\beta}\bigg)\big(\nabla f(\Xb_t, \bxi_t)-\nabla f(\Xb_{t-1}, \bxi_{t-1})\big),~\label{eq:moonlight_c_t}\\
    \Mb_t &= \beta \Mb_{t-1} + (1-\beta)\mathbf{C}_t.~\label{eq:moonlight_m_t}
\end{align}
 Comparing~\eqref{eq:approx_c_t} with~\eqref{eq:moonlight_c_t}, it can be seen that the approximate version of MARS-M corresponds to Moonlight with the stochastic gradient $\nabla f(\Xb_t, \bxi_t)$ being scaled down by a factor of $\color{blue}(1-\beta)$ and setting $\gamma_t = {\color{blue}1}$. The experiments in Section~\ref{sec:exp} empirically demonstrate that this difference in momentum parameters leads to a noticeable impact on the loss. 


\subsection{Convergence Analysis}
To analyze the convergence of MARS-M, similar to~\citet{yuan2024mars}, we first make the following assumptions:
\begin{assumption}[Bounded Variance]\label{assum:variance}
We assume that the variance of gradient estimator is bounded by $\sigma^2$. i.e., for any noise $\bxi$, parameter $\Xb$, and $\nabla F(\Xb) = \EE [\nabla f(\Xb, \bxi)]$, there exists a positive $\sigma$ such that:
\begin{align}
    \EE\big[\|\nabla f(\Xb, \bxi)-\nabla F(\Xb)\|_F^2\big]\le \sigma^2.
\end{align}
\end{assumption}

\begin{assumption}[$L$-Smoothness]\label{assum:L-smooth}
We assume that for arbitrary $\bxi$, $f(\Xb, \bxi)$ is $L$-smooth:
\begin{align}
    \|\nabla f(\Xb, \bxi)-\nabla f(\Yb, \bxi)\|_F\le L\cdot\|\Xb-\Yb\|_F,~\forall \Xb,\Yb.
\end{align}
\end{assumption}

Both assumptions are standard in the literature \citep{STORM,yuan2024mars}. We have the following theorem, which guarantees the convergence of MARS-M in Algorithm~\ref{alg:MARS_moonlight}.

\begin{theorem}\label{thm:mars-moonlight-convergence} In Algorithm~\ref{alg:MARS_moonlight}, under Assumptions~\ref{assum:variance} and~\ref{assum:L-smooth}, when choosing $\lambda=0$, $\eta_t=(s+t)^{-2/3}, s\geq 2$, suppose $ \beta_{t+1}= 1 - 2\eta_t$, then for $\forall T\ge s$, it holds that
\begin{align*}
    \frac{1}{T}\sum_{t=1}^T\EE\|\nabla F(\Xb_t)\|_F\le&\frac{2\sqrt{2LG}}{T^{1/3}} + \frac{2\sqrt{2LB\log(s+T)}}{T^{1/3}}\\
    &\qquad+\frac{2G}{T^{1/3}}+\frac{2B}{T^{1/3}}\log(s+T)\\&\qquad-
    \frac{\sqrt{2}}{4LT^{1/3}}\sum_{t=1}^T\frac{\tilde{M}_{t+1}}{\sqrt{\eta_t}}.
\end{align*}
where $G=F(\Xb_1)-\min_\Xb F(\Xb)+\frac{\sqrt{2}s^{1/3}\sigma^2}{4L}+\frac{3Ln}{2s^{1/3}}$, $B=\Big(\frac{2\sqrt{2} \sigma^2}{L}+\frac{3\sqrt{2}Ln}{2}\Big)$ and $\tilde{M}_{t+1}$ is a non-negative value defined in \eqref{eq:M_value}.
\end{theorem}

We defer the proof of the above theorem to Appendix~\ref{sec:proof_of_thm}. It is worth noting that time-varying momentum parameter $\beta_t$ is required for theoretical analysis, while it can be chosen as a constant in practice. Theorem~\ref{thm:mars-moonlight-convergence}  suggests that MARS-M can achieve a non-asymptotic convergence rate of $\mathcal{O}(T^{-1/3}\log(T))$, the same rate as general MARS as proved in~\citet{yuan2024mars}. 
As a comparison,~\citet{li2025note},~\citet{shen2025convergence} and~\citet{pethick2025training} proved that Muon can achieve only $\mathcal{O}(T^{-1/4})$ convergence rate. 

To sum up, MARS-M constitutes a new addition to the MARS family, broadening the framework to encompass more matrix-based optimizers with strong theoretical guarantees.

\begin{table*}[htb!]
\centering
\caption{The evaluation results of medium models pre-trained using the FineWeb-Edu 100B dataset (2-shot with lm-evaluation-harness). The best scores in each column are bolded. Abbreviations: WG = WinoGrande.}
\label{tab:medium-2-shot-fw}
\resizebox{\textwidth}{!}{%
\begin{tabular}{cccccccccc}
\hline
Method & ARC-C & ARC-E & Hellaswag & MMLU & OpenBookQA & PIQA & SciQ & WG & Avg. \\ \hline
AdamW & 33.53 & 66.46 & 45.02 & \textbf{25.32} & \textbf{35.40} & 69.10 & 86.70 & \textbf{55.80} & 52.17\\
Muon (Moonlight) & 33.45 & 66.33 & 45.48 & 24.73 & \textbf{35.40} & 69.21 & 89.00 & 54.06 & 52.21 \\
MARS-M ($\gamma=0.01$) & 34.73 & 66.12 & \textbf{46.23} & 24.56 & 34.00 & \textbf{70.40} & 88.70 & 53.83 & 52.32 \\
MARS-M ($\gamma=0.025$) & \textbf{35.24} & \textbf{67.30} & 45.82 & 24.98 & 33.40 & 68.28 & \textbf{89.10} & 54.93 & \textbf{52.38} \\ \hline
\end{tabular}%
}
\end{table*}

\begin{table*}[htb!]
\centering
\caption{The evaluation results of large models pre-trained using the FineWeb-Edu 100B dataset (2-shot with lm-evaluation-harness). The best scores in each column are bolded. Abbreviations: WG = WinoGrande.}
\label{tab:large-2-shot-fw}
\resizebox{\textwidth}{!}{%
\begin{tabular}{cccccccccc}
\hline
Method & ARC-C & ARC-E & Hellaswag & MMLU & OpenBookQA & PIQA & SciQ & WG & Avg. \\ \hline
AdamW & 38.05 & 70.29 & 50.30 & \textbf{26.87} & 38.20 & 70.46 & 92.10 & 55.80 & 55.26\\
Muon (Moonlight) & 37.03 & 70.08 & 51.57 & 25.32 & 38.20 & 72.03 & 91.00 & 56.43 & 55.21 \\
MARS-M ($\gamma=0.01$) & \textbf{39.08} & 70.16 & 52.07 & 24.90 & 37.60 & \textbf{72.25} & \textbf{93.30} & \textbf{56.91} & \textbf{55.78} \\
MARS-M ($\gamma=0.025$) & 37.03 & \textbf{72.01} & \textbf{52.22} & 25.10 & \textbf{39.40} & 71.22 & 90.90 & 55.64 & 55.41 \\ \hline
\end{tabular}%
}
\end{table*}

\begin{table*}[htb!]
\centering
\caption{The evaluation results of XL models pre-trained using the FineWeb-Edu 100B dataset (2-shot with lm-evaluation-harness). The best scores in each column are bolded. Abbreviations: WG = WinoGrande.}
\label{tab:xl-2-shot-fw}
\resizebox{\textwidth}{!}{%
\begin{tabular}{cccccccccc}
\hline
Method & ARC-C & ARC-E & Hellaswag & MMLU & OpenBookQA & PIQA & SciQ & WG & Avg. \\ \hline
Muon (Moonlight) & 40.87 & 71.21 & 55.80 & 24.89 & \textbf{42.40} & 74.05 & 91.90 & 57.54 & 57.33 \\
MARS-M ($\gamma=0.01$) & 41.64 & \textbf{73.57} & 57.01 & 24.75 & 39.60 & 73.99 & \textbf{92.30} & \textbf{58.96} & \textbf{57.73} \\
MARS-M ($\gamma=0.025$) & \textbf{42.66} & 72.10 & \textbf{57.09} & \textbf{25.57} & 39.60 & \textbf{74.21} & \textbf{92.30} & 58.09 & 57.70 \\ \hline
\end{tabular}%
}
\end{table*}

\section{Experiments}
\label{sec:exp}
\subsection{LLM Experiments}
We evaluate the performances of our algorithm with baseline algorithms\footnote{For training efficiency, we use approximated MARS-M for the LLM experiments, since the performances of exact version and approximate versions performs similarly as discussed in~\citet{yuan2024mars}. And we compare their differences in computer vision experiments.}, including Moonlight and AdamW, on language modeling tasks based on the nanoGPT~\citep{Karpathy2022} architecture and GPT-2 \citep{radford2019language} model. We conduct experiments on OpenWebText~\citep{Gokaslan2019OpenWeb} and FineWeb-Edu 100B~\citep{lozhkov2024fineweb-edu} datasets.
For OpenWebText, the training and validation datasets contain approximately $9$ billion and $4.4$ million tokens, respectively; while the training and validation sets of FineWeb-Edu 100B are with 100 billion and 0.1 billion tokens. We train for 100,000 steps with 2,000 warm-up steps, using a context length of 1024 and a total batch size of 480. 

We run experiments at four scales: small (125M parameters), medium (355M parameters), large (770M parameters) as well as XL (1.5B parameters). Additionally, we disable biases and set the dropout rate \citep{srivastava2014dropout} to $0.0$. And we also use Cosine learning rate scheduler and set the gradient clipping threshold to 1.0. For training parameters, for experiments with either OpenWebText or FineWeb-Edu 100B datasets, we perform a grid search over learning rates between $\{5e-4, 1e-3, 3e-3, 5e-3, 6e-3, 1e-2\}$ for Moonlight optimizer (the detailed learning rates are listed in Table~\ref{tb:hyperparameter-train} in Appendix~\ref{appendix_hyper}). We just apply the same learning rate for MARS-M optimizer. And we use $\beta=0.95$ for both Moonlight and MARS-M optimizers. Since Moonlight and MARS-M optimizers are designed only for matrix parameters, we optimize the vector-like parameters and embeddings with AdamW with the same learning rate.

For experiments with small models, we use 16 NVIDIA H800 GPUs; and for large models, we implement with 32 NVIDIA H800. Other training hyper-parameters are listed in Appendix~\ref{appendix_hyper}. 

\subsection{Experiment Results}
We show the zoomed-in curves of training and validation losses for different sizes of models on OpenWebText and FineWeb-Edu 100B datasets in Figures~\ref{fig:train_open}-\ref{fig:val_open} and Figures~\ref{fig:train_fw}-\ref{fig:val_fw}, respectively. And the entire curves can be found in Appendix~\ref{sec:additional_exp}. We also plot the curves for the experiments trained with AdamW in~\citet{yuan2024mars}, and the corresponding learning rates is listed in Table~\ref{tb:hyperparameter-train}, which are also attained from grid research as described in~\citet{yuan2024mars}. 

It can be observed that the experiments trained with MARS-M display steady improvement on both training and validation losses over Moonlight. Additionally, although the loss for AdamW is lower in the middle of training due to a smaller maximum learning rate, model trained with MARS-M achieves lower losses than that of AdamW in the final phase. Since the learning rates for all the experiments trained with baseline optimizers results from grid search, it can be concluded that MARS-M can indeed improve the performance of large language model training.

\begin{figure*}[htb!]
    \centering
    \includegraphics[width=0.32\linewidth]{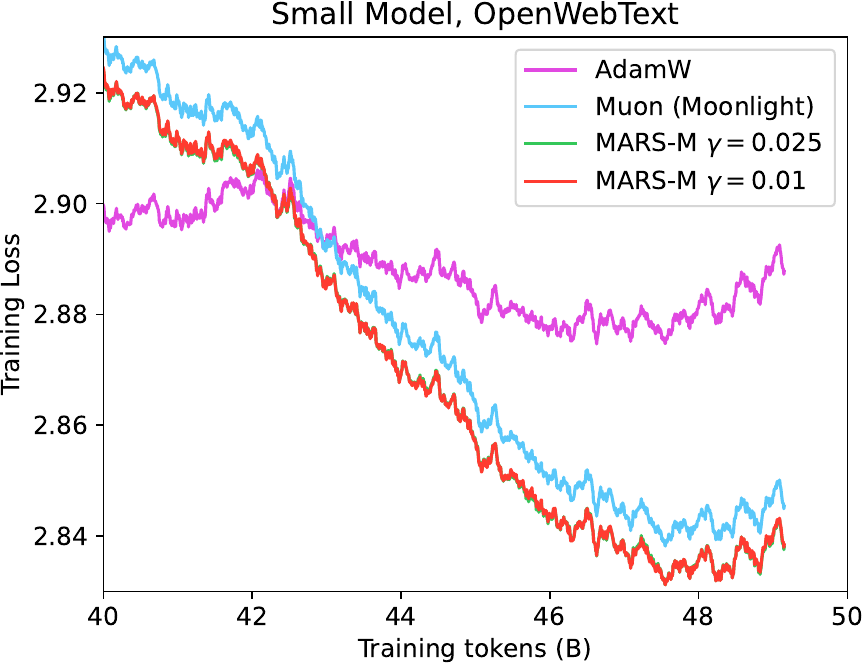}
    \includegraphics[width=0.32\linewidth]{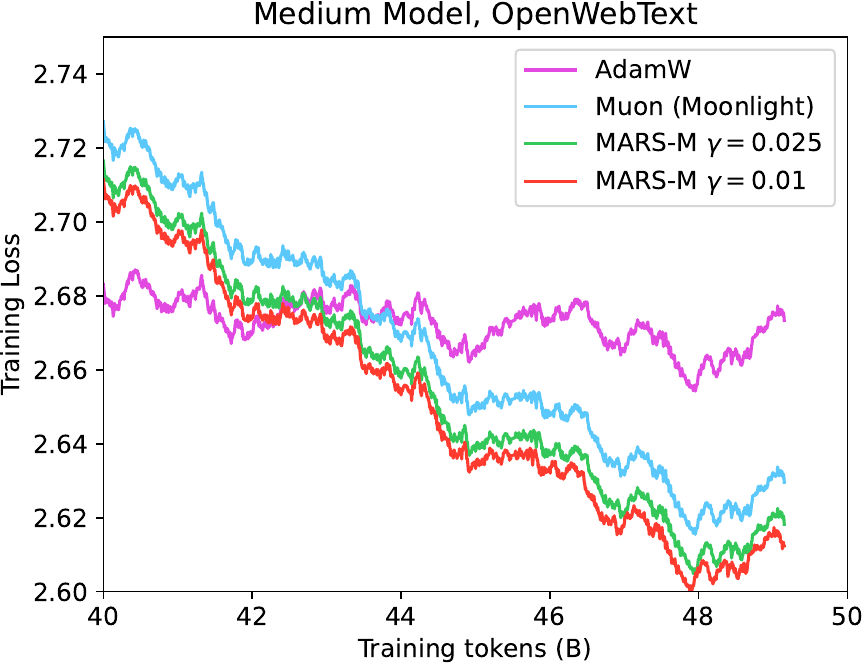}
    \includegraphics[width=0.32\linewidth]{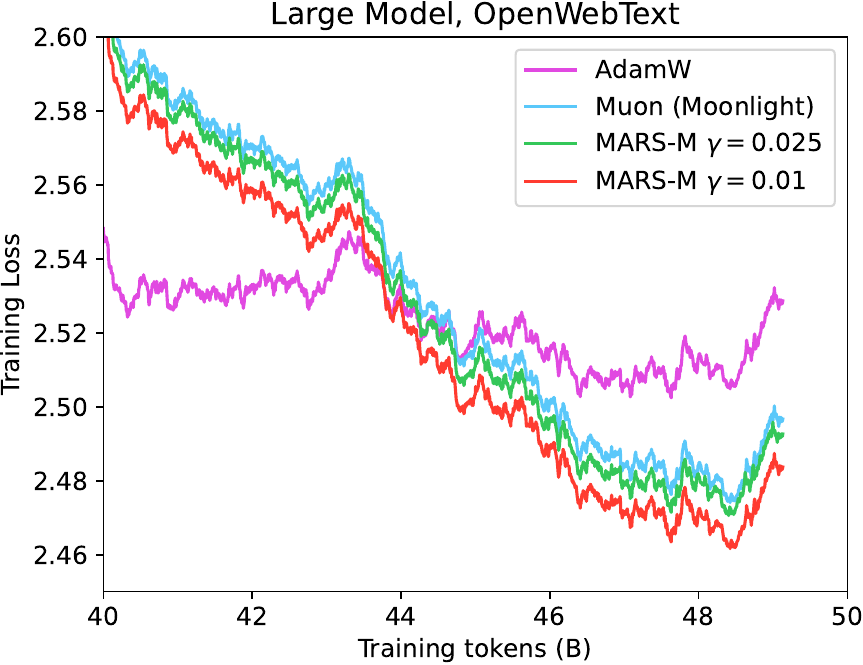}
    \caption{The zoomed-in training loss of small-size (125M), medium-size (355M) and large-size (770M) models trained with different optimizers on the OpenWebText dataset.}
    \label{fig:train_open}
\end{figure*}
\begin{figure*}[htb!]
    \centering
    \includegraphics[width=0.32\linewidth]{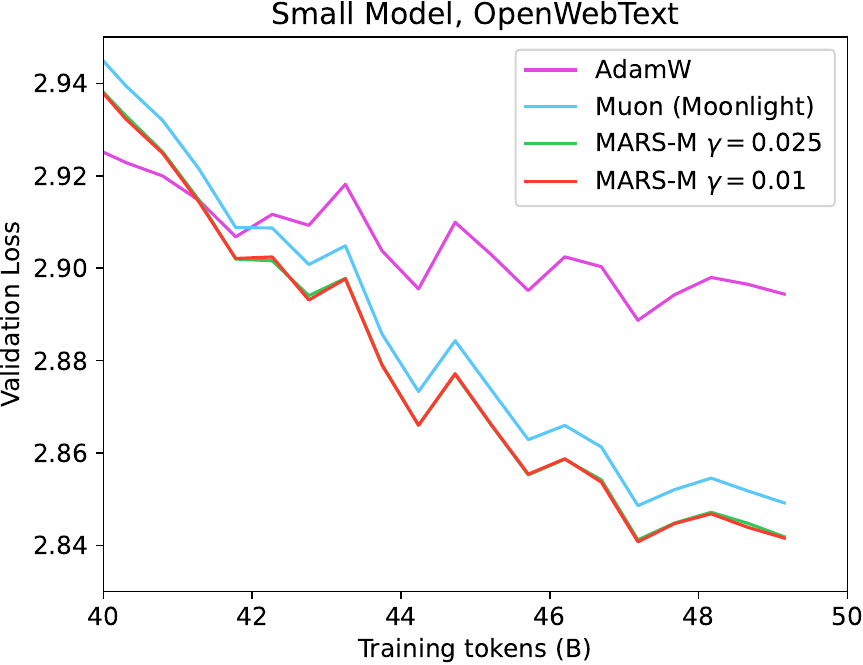}
    \includegraphics[width=0.32\linewidth]{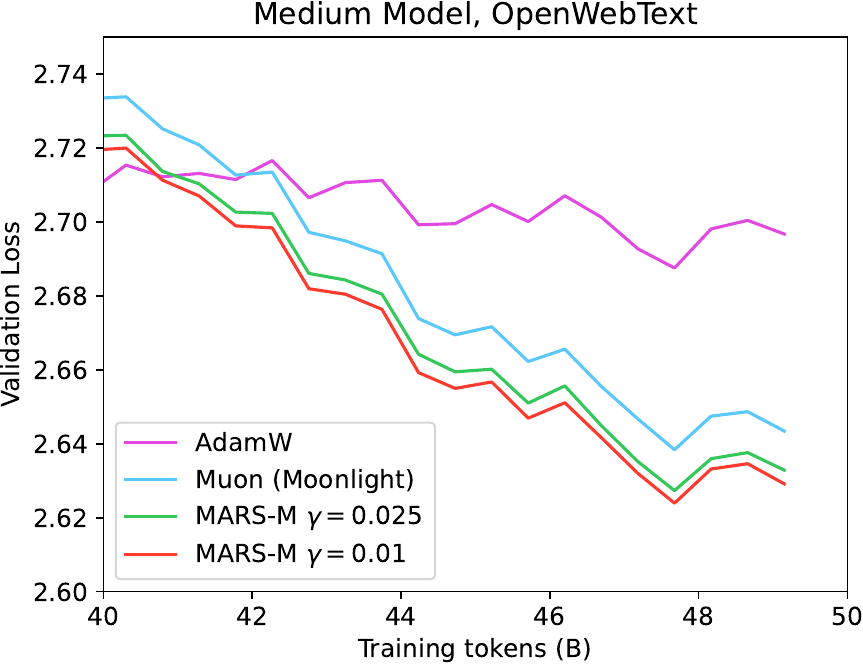}
    \includegraphics[width=0.32\linewidth]{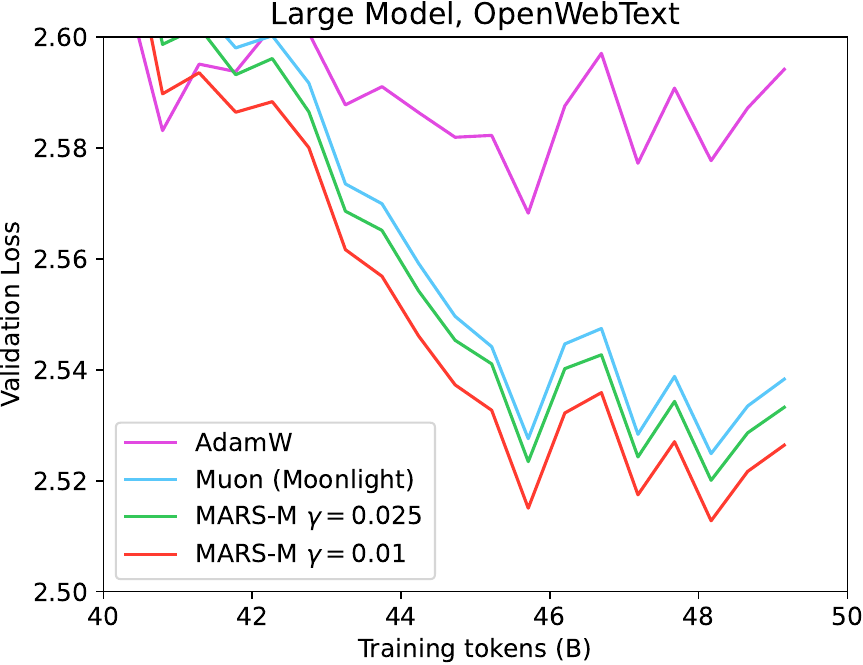}
    \caption{The zoomed-in training and validation loss of large-size models (770M) trained with different optimizers on the OpenWebText dataset.}
    \label{fig:val_open}
\end{figure*}
\begin{figure*}[htb!]
    \centering
    \includegraphics[width=0.32\linewidth]{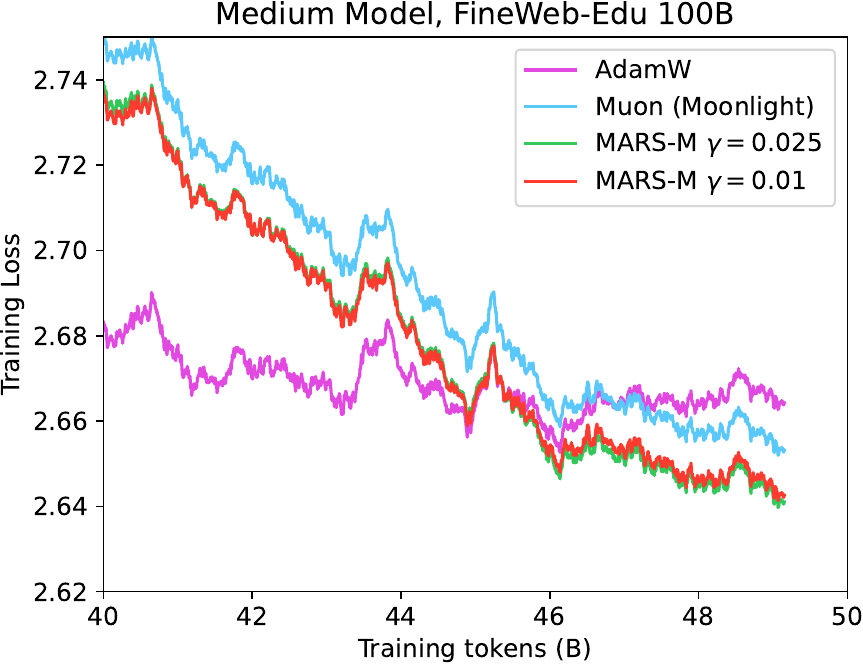}
    \includegraphics[width=0.32\linewidth]{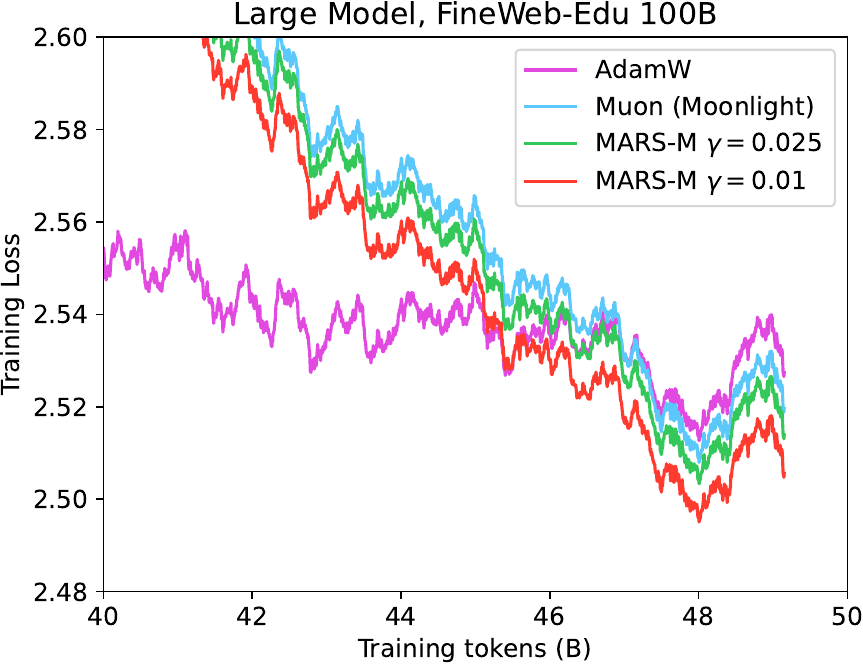}
    \includegraphics[width=0.32\linewidth]{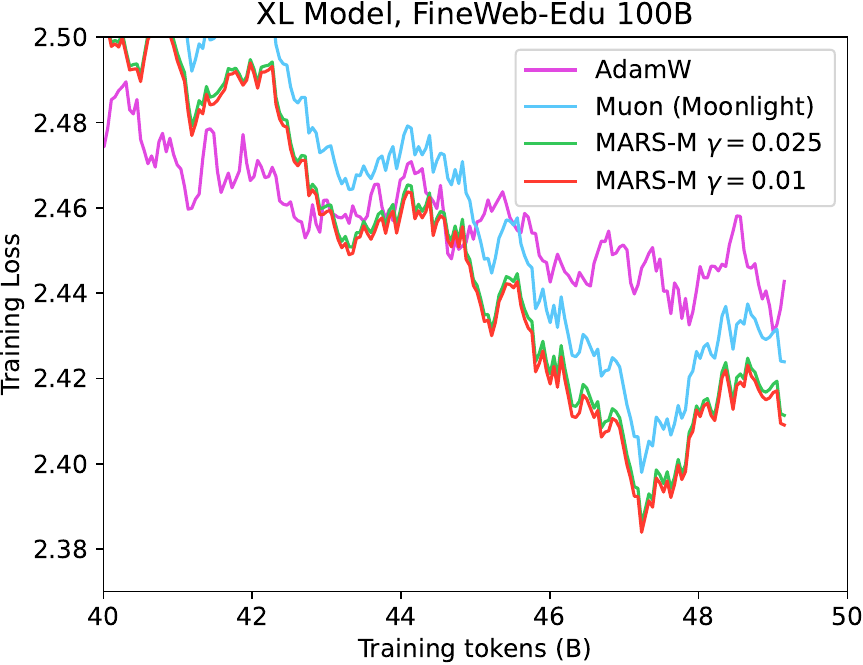}
    \caption{The zoomed-in training loss of medium-size (355M), large-size (770M) and XL-size (1.5B) models trained with different optimizers on the FineWeb-Edu 100B dataset.}
    \label{fig:train_fw}
\end{figure*}
\begin{figure*}[htb!]
    \centering
    \includegraphics[width=0.32\linewidth]{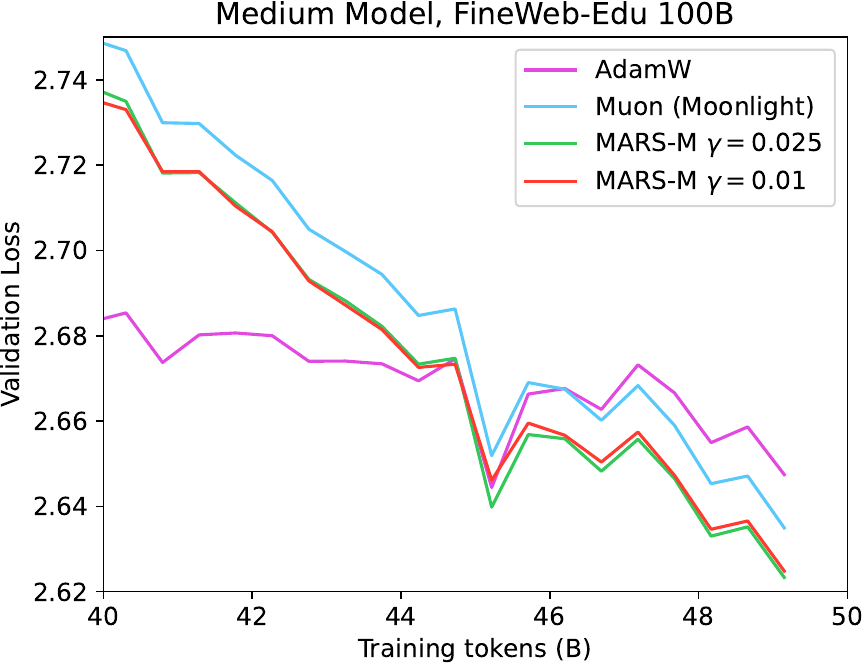}
    \includegraphics[width=0.32\linewidth]{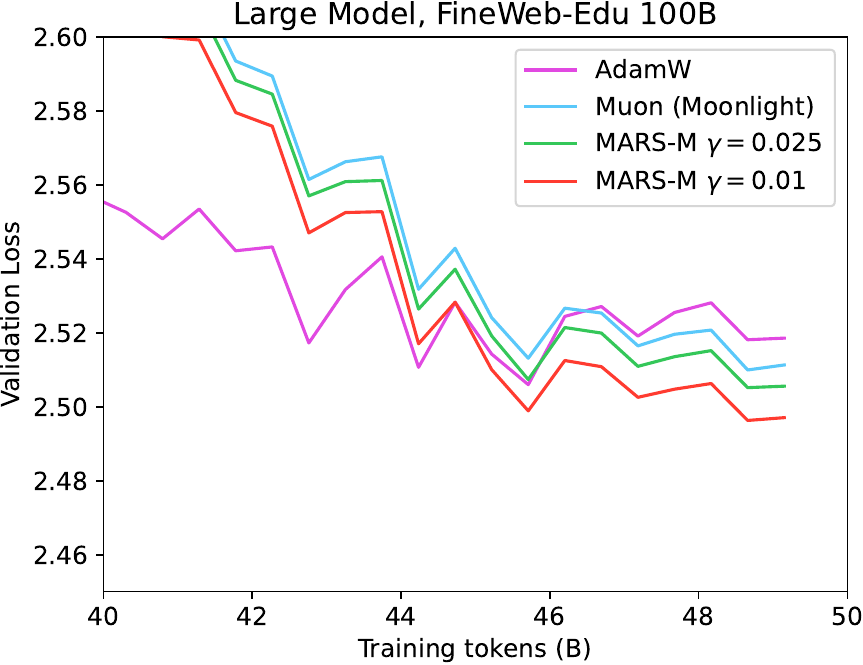}
    \includegraphics[width=0.32\linewidth]{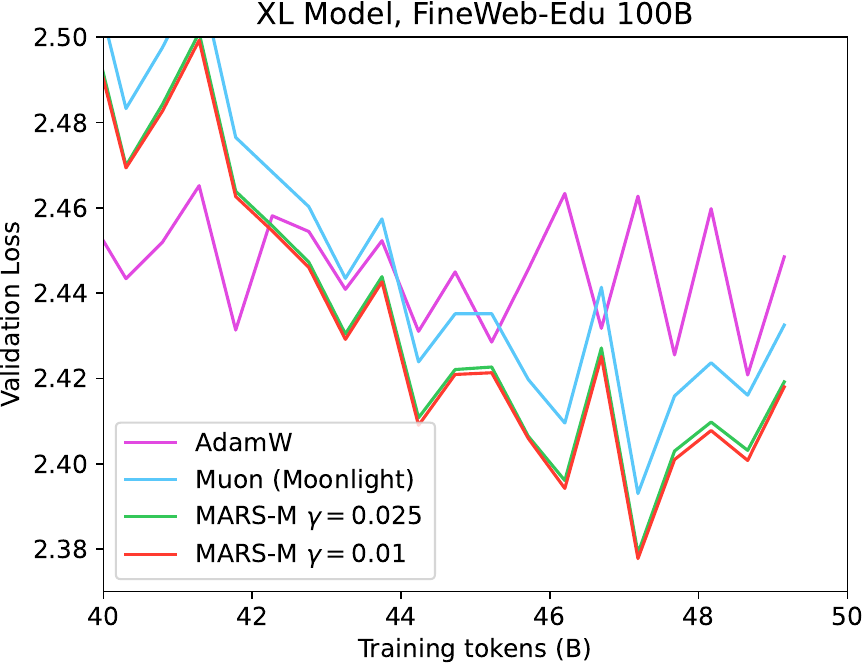}
    \caption{The zoomed-in validation loss of medium-size (355M), large-size (770M) and XL-size (1.5B) models trained with different optimizers on the FineWeb-Edu 100B dataset.}
    \label{fig:val_fw}
\end{figure*}

Furthermore, We also evaluate 0-shot and 2-shot performances of these optimizers on benchmarks including ARC~\citep{ARC}, HellaSwag~\citep{hellaswag}, OBQA~\citep{OpenBookQA}, PIQA~\citep{PIQA}, SciQ~\citep{SciQ}, WinoGrande~\citep{WinoGrande} and MMLU~\citep{MMLU}, based on \texttt{lm-evaluation-harness} codebase~\citep{eval-harness}. We only display the 2-shot performances for medium, large and XL models trained on FineWeb-Edu 100B dataset in Tables~\ref{tab:medium-2-shot-fw},~\ref{tab:large-2-shot-fw} and~\ref{tab:xl-2-shot-fw}, respectively. And other results are postponed to the Appendix~\ref{sec:additional_exp}. It can be observed that MARS-M outperforms Moonlight on most of the benchmarks, showing that our algorithm can enhance the performance of pre-trained large language models on a wide range of downstream tasks.
\begin{figure*}[h]
    \centering
    \includegraphics[width=0.47\linewidth]{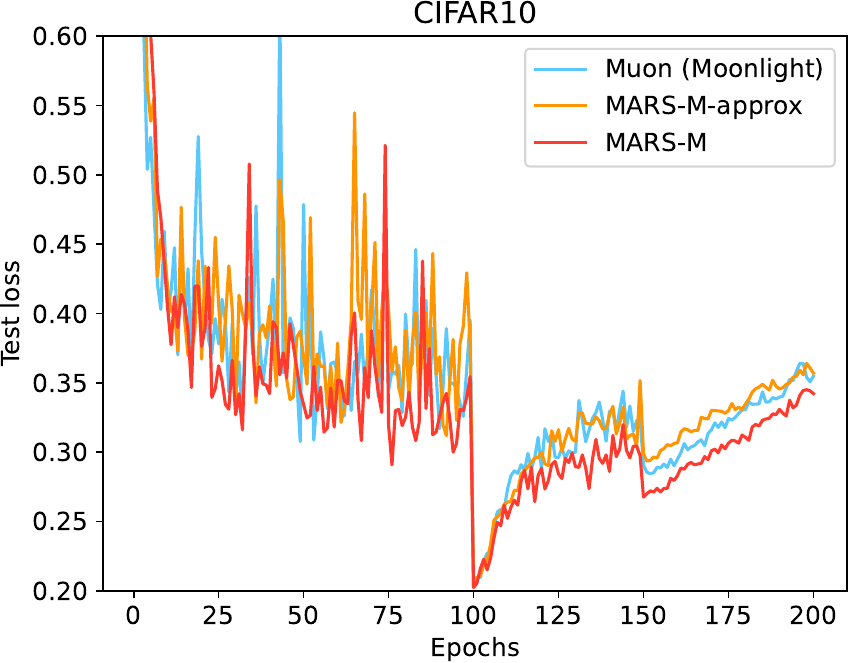}
    \includegraphics[width=0.458\linewidth]{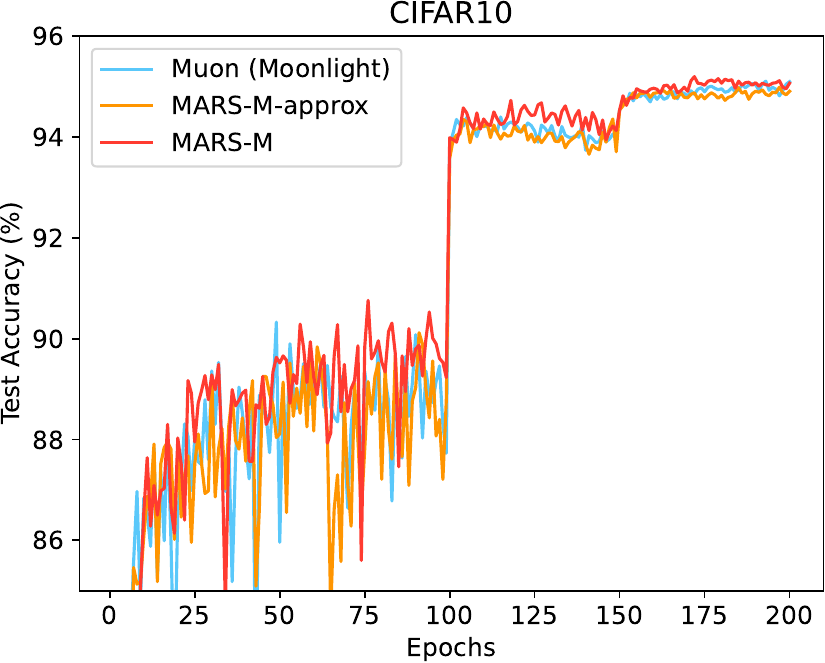}
    \caption{The test loss and test accuracy for different optimizers on CIFAR-10 dataset.}
    \label{fig:cifar10}
\end{figure*}

\begin{figure*}[h]
    \centering
    \includegraphics[width=0.47\linewidth]{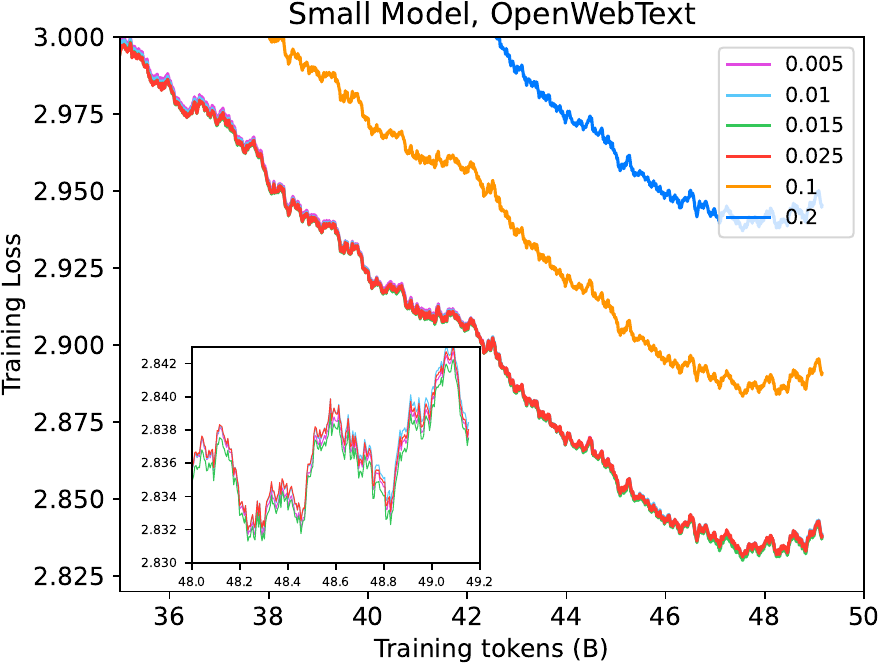}
    \includegraphics[width=0.464\linewidth]{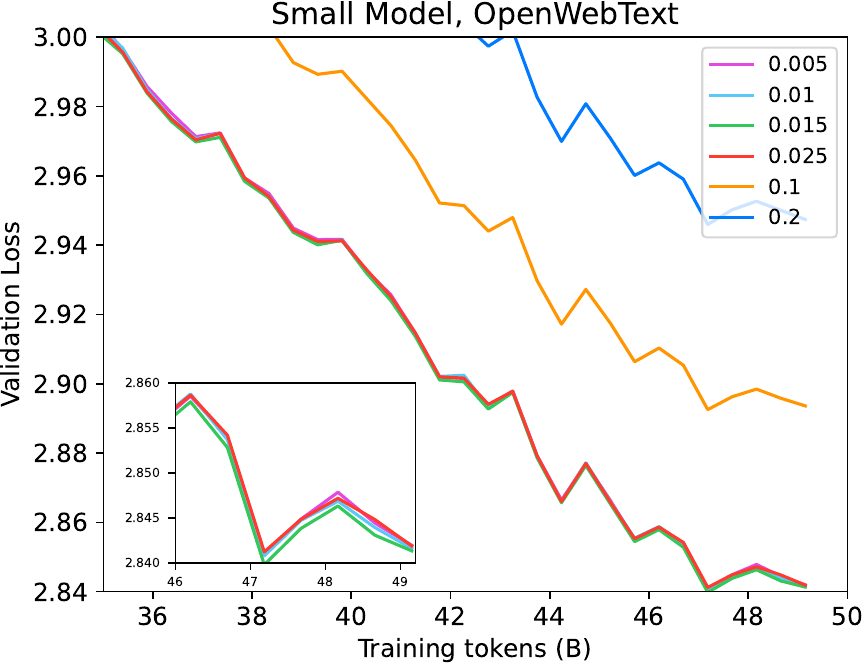}
    \caption{The zoomed-in training and validation loss of small-size models (125M) trained with MARS-M with different $\gamma$s on the OpenWebText 100B dataset.}
    \label{fig:gamma}
\end{figure*}
\subsection{Computer Vision Experiments}
Besides language model task, we also conduct experiments on computer vision tasks with these optimizers on CIFAR-10 dataset~\citep{krizhevsky2009learning} and ResNet-18 model~\citep{he2016deep}. We set $\gamma=0.025$ for MARS-M and do grid search over $\{10^{-4},\cdots,10^{-2}\}$ for the learning rate. We set $\beta=0.95$ for all the optimizers. All the experiments are done with a training batch size of 128 on 1 NVIDIA A6000 GPU, with a total of 200 training epochs. Additionally, we use MultiStepLR scheduler so that the learning rate will decrease to 10\% at 100th epoch and 1\% at 150th epoch. 

The test loss and accuracy for MARS-M and MARS-M approximate, as well as Moonlight optimizer, are shown in Figure~\ref{fig:cifar10}. It can be seen that MARS-M achieve better test accuracy and lower test loss over Moonlight and MARS-M-approx, validating the effect of variance reduction.

\subsection{Ablation Study}
We also implement an ablation study on $\gamma$ in MARS-M to investigate how the scale of gradient-correction term affects performance. The experiments are performed on small models trained on OpenWebText datasets, and the results are shown in Figure~\ref{fig:gamma}. It can be observed that the scale on gradient correction should be much smaller than 1, as usually taken in most of the variance reduction approaches~\citep{STORM,SARAH1,fang2018spider}. However, the training and validation losses for experiments with $\gamma\in[0.005, 0.025]$ are very close, suggesting that a small enough $\gamma$ is relatively robust to variation.

\section{Conclusion}
In this work, we have introduced MARS-M, a new algorithm that integrates MARS-based variance reduction with Muon/Moonlight to accelerate large language model training. Our algorithm integrates the benefits of variance reduction and matrix-based optimization. Through empirical experiments on language model and computer vision tasks, we demonstrate that models trained with MARS-M gain steady edges on loss and evaluation on downstream tasks over Moonlight baseline. In summary, our algorithm demonstrates an effective pathway for improving state-of-the-art matrix-based optimizers using variance reduction techniques.

\bibliography{reference}

@misc{eval-harness,
  author       = {Gao, Leo and Tow, Jonathan and Abbasi, Baber and Biderman, Stella and Black, Sid and DiPofi, Anthony and Foster, Charles and Golding, Laurence and Hsu, Jeffrey and Le Noac'h, Alain and Li, Haonan and McDonell, Kyle and Muennighoff, Niklas and Ociepa, Chris and Phang, Jason and Reynolds, Laria and Schoelkopf, Hailey and Skowron, Aviya and Sutawika, Lintang and Tang, Eric and Thite, Anish and Wang, Ben and Wang, Kevin and Zou, Andy},
  title        = {A framework for few-shot language model evaluation},
  month        = 07,
  year         = 2024,
  publisher    = {Zenodo},
  version      = {v0.4.3},
  doi          = {10.5281/zenodo.12608602},
  url          = {https://zenodo.org/records/12608602}
}

@article{STORM,
  title={Momentum-based variance reduction in non-convex sgd},
  author={Cutkosky, Ashok and Orabona, Francesco},
  journal={Advances in neural information processing systems},
  volume={32},
  year={2019}
}

@inproceedings{Adamw,
  author       = {Ilya Loshchilov and
                  Frank Hutter},
  title        = {Decoupled Weight Decay Regularization},
  booktitle    = {7th International Conference on Learning Representations, {ICLR} 2019,
                  New Orleans, LA, USA, May 6-9, 2019},
  year         = {2019},
}

@inproceedings{adam,
  author       = {Diederik P. Kingma and
                  Jimmy Ba},
  editor       = {Yoshua Bengio and
                  Yann LeCun},
  title        = {Adam: {A} Method for Stochastic Optimization},
  booktitle    = {3rd International Conference on Learning Representations, {ICLR} 2015,
                  San Diego, CA, USA, May 7-9, 2015, Conference Track Proceedings},
  year         = {2015},
}

@article{adagrad,
  title={Adaptive subgradient methods for online learning and stochastic optimization.},
  author={Duchi, John and Hazan, Elad and Singer, Yoram},
  journal={Journal of machine learning research},
  volume={12},
  number={7},
  year={2011}
}

@article{llama3,
  title={The llama 3 herd of models},
  author={Dubey, Abhimanyu and Jauhri, Abhinav and Pandey, Abhinav and Kadian, Abhishek and Al-Dahle, Ahmad and Letman, Aiesha and Mathur, Akhil and Schelten, Alan and Yang, Amy and Fan, Angela and others},
  journal={arXiv preprint arXiv:2407.21783},
  year={2024}
}

@article{johnson2013accelerating,
  title={Accelerating stochastic gradient descent using predictive variance reduction},
  author={Johnson, Rie and Zhang, Tong},
  journal={Advances in neural information processing systems},
  volume={26},
  year={2013}
}

@article{fang2018spider,
  title={Spider: Near-optimal non-convex optimization via stochastic path-integrated differential estimator},
  author={Fang, Cong and Li, Chris Junchi and Lin, Zhouchen and Zhang, Tong},
  journal={Advances in neural information processing systems},
  volume={31},
  year={2018}
}

@article{zhou2020stochastic,
  title={Stochastic nested variance reduction for nonconvex optimization},
  author={Zhou, Dongruo and Xu, Pan and Gu, Quanquan},
  journal={Journal of machine learning research},
  volume={21},
  number={103},
  pages={1--63},
  year={2020}
}

@inproceedings{SARAH1,
  title={SARAH: A novel method for machine learning problems using stochastic recursive gradient},
  author={Nguyen, Lam M and Liu, Jie and Scheinberg, Katya and Tak{\'a}{\v{c}}, Martin},
  booktitle={International conference on machine learning},
  pages={2613--2621},
  year={2017},
  organization={PMLR}
}

@article{SARAH2,
  title={Stochastic recursive gradient algorithm for nonconvex optimization},
  author={Nguyen, Lam M and Liu, Jie and Scheinberg, Katya and Tak{\'a}{\v{c}}, Martin},
  journal={arXiv preprint arXiv:1705.07261},
  year={2017}
}

@article{levy2021storm+,
  title={Storm+: Fully adaptive sgd with recursive momentum for nonconvex optimization},
  author={Levy, Kfir and Kavis, Ali and Cevher, Volkan},
  journal={Advances in Neural Information Processing Systems},
  volume={34},
  pages={20571--20582},
  year={2021}
}

@article{krizhevsky2009learning,
  title={Learning multiple layers of features from tiny images},
  author={Krizhevsky, Alex and Hinton, Geoffrey and others},
  year={2009},
  publisher={Toronto, ON, Canada}
}

@inproceedings{he2016deep,
  title={Deep residual learning for image recognition},
  author={He, Kaiming and Zhang, Xiangyu and Ren, Shaoqing and Sun, Jian},
  booktitle={Proceedings of the IEEE conference on computer vision and pattern recognition},
  pages={770--778},
  year={2016}
}

@article{radford2019language,
  title={Language models are unsupervised multitask learners},
  author={Radford, Alec and Wu, Jeffrey and Child, Rewon and Luan, David and Amodei, Dario and Sutskever, Ilya and others},
  journal={OpenAI blog},
  volume={1},
  number={8},
  pages={9},
  year={2019}
}

@misc{Gokaslan2019OpenWeb,
    title={OpenWebText Corpus},
    author={Gokaslan, Aaron and Cohen, Vanya and Pavlick, Ellie and Tellex, Stefanie},
    howpublished={\url{http://Skylion007.github.io/OpenWebTextCorpus}},
    year={2019}
}

@article{Lion,
  title={Symbolic discovery of optimization algorithms},
  author={Chen, Xiangning and Liang, Chen and Huang, Da and Real, Esteban and Wang, Kaiyuan and Pham, Hieu and Dong, Xuanyi and Luong, Thang and Hsieh, Cho-Jui and Lu, Yifeng and others},
  journal={Advances in neural information processing systems},
  volume={36},
  year={2023}
}

@article{liu2020adam,
  title={Adam $^+$: A Stochastic Method with Adaptive Variance Reduction},
  author={Liu, Mingrui and Zhang, Wei and Orabona, Francesco and Yang, Tianbao},
  journal={arXiv preprint arXiv:2011.11985},
  year={2020}
}

@misc{Karpathy2022,
  author = {Andrej Karpathy},
  title = {\text{NanoGPT}},
  year = {2022},
  publisher = {GitHub},
  journal = {GitHub repository},
  howpublished = {\url{https://github.com/karpathy/nanoGPT}},
  commit = {325be85d9be8c81b436728a420e85796c57dba7e}
}

@article{srivastava2014dropout,
  title={Dropout: a simple way to prevent neural networks from overfitting},
  author={Srivastava, Nitish and Hinton, Geoffrey and Krizhevsky, Alex and Sutskever, Ilya and Salakhutdinov, Ruslan},
  journal={The journal of machine learning research},
  volume={15},
  number={1},
  pages={1929--1958},
  year={2014},
  publisher={JMLR. org}
}

@misc{muon,
  author       = {Keller Jordan and Yuchen Jin and Vlado Boza and You Jiacheng and
                  Franz Cecista and Laker Newhouse and Jeremy Bernstein},
  title        = {Muon: An optimizer for hidden layers in neural networks},
  year         = {2024},
  url          = {https://kellerjordan.github.io/posts/muon/}
}

@book{higham2008functions,
  title={Functions of Matrices},
  author={Higham, Nicholas J.},
  year={2008},
  publisher={Society for Industrial and Applied Mathematics}
}

@article{schulz1933iterative,
  title={Iterative Berechnung der reziproken Matrix},
  author={Schulz, G{\"u}nther},
  journal={Z. Angew. Math. Mech.},
  volume={13},
  pages={57--59},
  year={1933}
}

@inproceedings{bernstein2024old,
  title={Old Optimizer, New Norm: An Anthology},
  author={Bernstein, Jeremy and Newhouse, Laker},
  booktitle={OPT 2024: Optimization for Machine Learning},
  year={2024}
}

@inproceedings{gupta2018shampoo,
  title={Shampoo: Preconditioned stochastic tensor optimization},
  author={Gupta, Vineet and Koren, Tomer and Singer, Yoram},
  booktitle={International Conference on Machine Learning},
  pages={1842--1850},
  year={2018},
  organization={PMLR}
}

@article{lakic1998computation,
  title={On the computation of the matrix k-th root},
  author={Laki{\'c}, Slobodan},
  journal={ZAMM-Journal of Applied Mathematics and Mechanics/Zeitschrift f{\"u}r Angewandte Mathematik und Mechanik: Applied Mathematics and Mechanics},
  volume={78},
  number={3},
  pages={167--172},
  year={1998},
  publisher={Wiley Online Library}
}

@article{anil2020scalable,
  title={Scalable second order optimization for deep learning},
  author={Anil, Rohan and Gupta, Vineet and Koren, Tomer and Regan, Kevin and Singer, Yoram},
  journal={arXiv preprint arXiv:2002.09018},
  year={2020}
}

@book{nesterov2013introductory,
  title={Introductory lectures on convex optimization: A basic course},
  author={Nesterov, Yurii},
  volume={87},
  year={2013},
  publisher={Springer Science \& Business Media}
}

@misc{lozhkov2024fineweb-edu,
    author       = { Lozhkov, Anton and Ben Allal, Loubna and von Werra, Leandro and Wolf, Thomas },  
    title        = { FineWeb-Edu: the Finest Collection of Educational Content }, 
    year         = 2024,  
    url          = { https://huggingface.co/datasets/HuggingFaceFW/fineweb-edu },  
    doi          = { 10.57967/hf/2497 },
    publisher    = { Hugging Face }
}

@article{defazio2014saga,
  title={SAGA: A fast incremental gradient method with support for non-strongly convex composite objectives},
  author={Defazio, Aaron and Bach, Francis and Lacoste-Julien, Simon},
  journal={Advances in neural information processing systems},
  volume={27},
  year={2014}
}

@inproceedings{yin2023coefficient,
  title={A Coefficient Makes SVRG Effective},
  author={Yin, Yida and Xu, Zhiqiu and Li, Zhiyuan and Darrell, Trevor and Liu, Zhuang},
  booktitle={The Thirteenth International Conference on Learning Representations},
  year={2024}
}

@article{kavis2022adaptive,
  title={Adaptive stochastic variance reduction for non-convex finite-sum minimization},
  author={Kavis, Ali and Skoulakis, Stratis and Antonakopoulos, Kimon and Dadi, Leello Tadesse and Cevher, Volkan},
  journal={Advances in Neural Information Processing Systems},
  volume={35},
  pages={23524--23538},
  year={2022}
}

@article{roux2012stochastic,
  title={A stochastic gradient method with an exponential convergence \_rate for finite training sets},
  author={Roux, Nicolas and Schmidt, Mark and Bach, Francis},
  journal={Advances in neural information processing systems},
  volume={25},
  year={2012}
}

@article{huang2021super,
  title={Super-adam: faster and universal framework of adaptive gradients},
  author={Huang, Feihu and Li, Junyi and Huang, Heng},
  journal={Advances in Neural Information Processing Systems},
  volume={34},
  pages={9074--9085},
  year={2021}
}

@article{hazan2015beyond,
  title={Beyond convexity: Stochastic quasi-convex optimization},
  author={Hazan, Elad and Levy, Kfir and Shalev-Shwartz, Shai},
  journal={Advances in neural information processing systems},
  volume={28},
  year={2015}
}

@article{defazio2019ineffectiveness,
  title={On the ineffectiveness of variance reduced optimization for deep learning},
  author={Defazio, Aaron and Bottou, L{\'e}on},
  journal={Advances in Neural Information Processing Systems},
  volume={32},
  year={2019}
}

@inproceedings{shazeer2018adafactor,
  title={Adafactor: Adaptive learning rates with sublinear memory cost},
  author={Shazeer, Noam and Stern, Mitchell},
  booktitle={International Conference on Machine Learning},
  pages={4596--4604},
  year={2018},
  organization={PMLR}
}

@inproceedings{allen2016improved,
  title={Improved SVRG for non-strongly-convex or sum-of-non-convex objectives},
  author={Allen-Zhu, Zeyuan and Yuan, Yang},
  booktitle={International conference on machine learning},
  pages={1080--1089},
  year={2016},
  organization={PMLR}
}

@inproceedings{reddi2016stochastic,
  title={Stochastic variance reduction for nonconvex optimization},
  author={Reddi, Sashank J and Hefny, Ahmed and Sra, Suvrit and Poczos, Barnabas and Smola, Alex},
  booktitle={International conference on machine learning},
  pages={314--323},
  year={2016},
  organization={PMLR}
}

@article{shalev2013stochastic,
  title={Stochastic dual coordinate ascent methods for regularized loss minimization.},
  author={Shalev-Shwartz, Shai and Zhang, Tong},
  journal={Journal of Machine Learning Research},
  volume={14},
  number={1},
  year={2013}
}

@article{wang2019spiderboost,
  title={Spiderboost and momentum: Faster variance reduction algorithms},
  author={Wang, Zhe and Ji, Kaiyi and Zhou, Yi and Liang, Yingbin and Tarokh, Vahid},
  journal={Advances in Neural Information Processing Systems},
  volume={32},
  year={2019}
}

@inproceedings{SOAP,
  year={2024},
  title={SOAP: Improving and Stabilizing Shampoo using Adam for Language Modeling},
  author={Vyas, Nikhil and Morwani, Depen and Zhao, Rosie and Shapira, Itai and Brandfonbrener, David and Janson, Lucas and Kakade, Sham M},
  booktitle={The Thirteenth International Conference on Learning Representations}
}

@inproceedings{hellaswag,
  author       = {Rowan Zellers and
                  Ari Holtzman and
                  Yonatan Bisk and
                  Ali Farhadi and
                  Yejin Choi},
  editor       = {Anna Korhonen and
                  David R. Traum and
                  Llu{\'{\i}}s M{\`{a}}rquez},
  title        = {HellaSwag: Can a Machine Really Finish Your Sentence?},
  booktitle    = {Proceedings of the 57th Conference of the Association for Computational
                  Linguistics, {ACL} 2019, Florence, Italy, July 28- August 2, 2019,
                  Volume 1: Long Papers},
  pages        = {4791--4800},
  publisher    = {Association for Computational Linguistics},
  year         = {2019},
  doi          = {10.18653/V1/P19-1472},
}

@inproceedings{MMLU,
  author       = {Dan Hendrycks and
                  Collin Burns and
                  Steven Basart and
                  Andy Zou and
                  Mantas Mazeika and
                  Dawn Song and
                  Jacob Steinhardt},
  title        = {Measuring Massive Multitask Language Understanding},
  booktitle    = {9th International Conference on Learning Representations, {ICLR} 2021,
                  Virtual Event, Austria, May 3-7, 2021},
  year         = {2021},
}

@inproceedings{WinoGrande, 
  author       = {Keisuke Sakaguchi and
                  Ronan Le Bras and
                  Chandra Bhagavatula and
                  Yejin Choi},
  title        = {WinoGrande: An Adversarial Winograd Schema Challenge at Scale},
  booktitle    = {The Thirty-Fourth {AAAI} Conference on Artificial Intelligence, {AAAI}
                  2020, The Thirty-Second Innovative Applications of Artificial Intelligence
                  Conference, {IAAI} 2020, The Tenth {AAAI} Symposium on Educational
                  Advances in Artificial Intelligence, {EAAI} 2020, New York, NY, USA,
                  February 7-12, 2020},
  pages        = {8732--8740},
  publisher    = {{AAAI} Press},
  year         = {2020},
}

@inproceedings{PIQA,
  author       = {Yonatan Bisk and
                  Rowan Zellers and
                  Ronan Le Bras and
                  Jianfeng Gao and
                  Yejin Choi},
  title        = {{PIQA:} Reasoning about Physical Commonsense in Natural Language},
  booktitle    = {The Thirty-Fourth {AAAI} Conference on Artificial Intelligence, {AAAI}
                  2020, The Thirty-Second Innovative Applications of Artificial Intelligence
                  Conference, {IAAI} 2020, The Tenth {AAAI} Symposium on Educational
                  Advances in Artificial Intelligence, {EAAI} 2020, New York, NY, USA,
                  February 7-12, 2020},
  pages        = {7432--7439},
  publisher    = {{AAAI} Press},
  year         = {2020},
}

@inproceedings{OpenBookQA,
  author       = {Todor Mihaylov and
                  Peter Clark and
                  Tushar Khot and
                  Ashish Sabharwal},
  editor       = {Ellen Riloff and
                  David Chiang and
                  Julia Hockenmaier and
                  Jun'ichi Tsujii},
  title        = {Can a Suit of Armor Conduct Electricity? {A} New Dataset for Open
                  Book Question Answering},
  booktitle    = {Proceedings of the 2018 Conference on Empirical Methods in Natural
                  Language Processing, Brussels, Belgium, October 31 - November 4, 2018},
  pages        = {2381--2391},
  publisher    = {Association for Computational Linguistics},
  year         = {2018},
  doi          = {10.18653/V1/D18-1260},
}

@inproceedings{ARC,
  author       = {Vikas Yadav and
                  Steven Bethard and
                  Mihai Surdeanu},
  editor       = {Kentaro Inui and
                  Jing Jiang and
                  Vincent Ng and
                  Xiaojun Wan},
  title        = {Quick and (not so) Dirty: Unsupervised Selection of Justification
                  Sentences for Multi-hop Question Answering},
  booktitle    = {Proceedings of the 2019 Conference on Empirical Methods in Natural
                  Language Processing and the 9th International Joint Conference on
                  Natural Language Processing, {EMNLP-IJCNLP} 2019, Hong Kong, China},
  pages        = {2578--2589},
  publisher    = {Association for Computational Linguistics},
  year         = {2019},
  doi          = {10.18653/V1/D19-1260},
}

@phdthesis{li2024smoothness,
  title={Smoothness and Adaptivity in Nonlinear Optimization for Machine Learning Applications},
  author={Li, Haochuan},
  year={2024},
  school={Massachusetts Institute of Technology}
}

@article{frangella2024promise,
  title={Promise: Preconditioned stochastic optimization methods by incorporating scalable curvature estimates},
  author={Frangella, Zachary and Rathore, Pratik and Zhao, Shipu and Udell, Madeleine},
  journal={Journal of Machine Learning Research},
  volume={25},
  number={346},
  pages={1--57},
  year={2024}
}

@inproceedings{derezinski2023stochastic,
  year={2023},
  title={Stochastic Variance-Reduced Newton: Accelerating Finite-Sum Minimization with Large Batches},
  author={Derezinski, Michal},
  booktitle={OPT 2023: Optimization for Machine Learning}
}

@inproceedings{yuan2024mars,
  title={MARS: Unleashing the Power of Variance Reduction for Training Large Models},
  author={Yuan, Huizhuo and Liu, Yifeng and Wu, Shuang and Gu, Quanquan and others},
  year={2025},
  booktitle={Forty-second International Conference on Machine Learning}
}

@article{liu2025muon,
  title={Muon is scalable for LLM training},
  author={Liu, Jingyuan and Su, Jianlin and Yao, Xingcheng and Jiang, Zhejun and Lai, Guokun and Du, Yulun and Qin, Yidao and Xu, Weixin and Lu, Enzhe and Yan, Junjie and others},
  journal={arXiv preprint arXiv:2502.16982},
  year={2025}
}

@article{semenov2025benchmarking,
  title={Benchmarking Optimizers for Large Language Model Pretraining},
  author={Semenov, Andrei and Pagliardini, Matteo and Jaggi, Martin},
  journal={arXiv preprint arXiv:2509.01440},
  year={2025}
}

@article{team2025kimi,
  title={Kimi k2: Open agentic intelligence},
  author={Team, Kimi and Bai, Yifan and Bao, Yiping and Chen, Guanduo and Chen, Jiahao and Chen, Ningxin and Chen, Ruijue and Chen, Yanru and Chen, Yuankun and Chen, Yutian and others},
  journal={arXiv preprint arXiv:2507.20534},
  year={2025}
}

@article{zeng2025glm,
  title={Glm-4.5: Agentic, reasoning, and coding (arc) foundation models},
  author={Zeng, Aohan and Lv, Xin and Zheng, Qinkai and Hou, Zhenyu and Chen, Bin and Xie, Chengxing and Wang, Cunxiang and Yin, Da and Zeng, Hao and Zhang, Jiajie and others},
  journal={arXiv preprint arXiv:2508.06471},
  year={2025}
}

@article{wen2025fantastic,
  title={Fantastic Pretraining Optimizers and Where to Find Them},
  author={Wen, Kaiyue and Hall, David and Ma, Tengyu and Liang, Percy},
  journal={arXiv preprint arXiv:2509.02046},
  year={2025}
}

@article{erdogdu2015convergence,
  title={Convergence rates of sub-sampled Newton methods},
  author={Erdogdu, Murat A and Montanari, Andrea},
  journal={Advances in Neural Information Processing Systems},
  volume={28},
  year={2015}
}

@article{gonen2015faster,
  title={Faster sgd using sketched conditioning},
  author={Gonen, Alon and Shalev-Shwartz, Shai},
  journal={arXiv preprint arXiv:1506.02649},
  year={2015}
}

@inproceedings{liusophia,
  title={Sophia: A Scalable Stochastic Second-order Optimizer for Language Model Pre-training},
  author={Liu, Hong and Li, Zhiyuan and Hall, David Leo Wright and Liang, Percy and Ma, Tengyu},
  booktitle={The Twelfth International Conference on Learning Representations},
  year={2024}
}

@inproceedings{yao2021adahessian,
  title={Adahessian: An adaptive second order optimizer for machine learning},
  author={Yao, Zhewei and Gholami, Amir and Shen, Sheng and Mustafa, Mustafa and Keutzer, Kurt and Mahoney, Michael},
  booktitle={proceedings of the AAAI conference on artificial intelligence},
  volume={35},
  number={12},
  pages={10665--10673},
  year={2021}
}

@inproceedings{riabinin2025gluon,
  title={Gluon: Making Muon \& Scion Great Again!(Bridging Theory and Practice of LMO-based Optimizers for LLMs)},
  author={Riabinin, Artem and Shulgin, Egor and Gruntkowska, Kaja and Richt{\'a}rik, Peter},
  booktitle={High-dimensional Learning Dynamics 2025},
  year={2025}
}

@inproceedings{pethick2025training,
  title={Training Deep Learning Models with Norm-Constrained LMOs},
  author={Pethick, Thomas and Xie, Wanyun and Antonakopoulos, Kimon and Zhu, Zhenyu and Silveti-Falls, Antonio and Cevher, Volkan},
  booktitle={Forty-second International Conference on Machine Learning},
  year={2025}
}

@article{guo2025deepseek,
  author       = {Daya Guo and
                  Dejian Yang and
                  Haowei Zhang and
                  Junxiao Song and
                  Peiyi Wang and
                  Qihao Zhu and
                  Runxin Xu and
                  Ruoyu Zhang and
                  Shirong Ma and
                  Xiao Bi and others},
  title        = {DeepSeek-R1 incentivizes reasoning in LLMs through reinforcement learning},
  journal      = {Nat.},
  volume       = {645},
  number       = {8081},
  pages        = {633--638},
  year         = {2025},
  doi          = {10.1038/S41586-025-09422-Z},
}

@misc{chatgpt,
  title = {ChatGPT},
  author = {{OpenAI}},
  howpublished = {\url{https://chat.openai.com/}},
  year = {2023},
}

@article{shen2025convergence,
  title={On the convergence analysis of muon},
  author={Shen, Wei and Huang, Ruichuan and Huang, Minhui and Shen, Cong and Zhang, Jiawei},
  journal={arXiv preprint arXiv:2505.23737},
  year={2025}
}

@article{li2025note,
  title={A Note on the Convergence of Muon},
  author={Li, Jiaxiang and Hong, Mingyi},
  journal={arXiv preprint arXiv:2502.02900},
  year={2025}
}

@article{lau2025polargrad,
  title={PolarGrad: A Class of Matrix-Gradient Optimizers from a Unifying Preconditioning Perspective},
  author={Lau, Tim Tsz-Kit and Long, Qi and Su, Weijie},
  journal={arXiv preprint arXiv:2505.21799},
  year={2025}
}

@article{sfyraki2025lions,
  title={Lions and muons: Optimization via stochastic frank-wolfe},
  author={Sfyraki, Maria-Eleni and Wang, Jun-Kun},
  journal={arXiv preprint arXiv:2506.04192},
  year={2025}
}

@inproceedings{Sato2025ConvergenceBA,
  title={Convergence Bound and Critical Batch Size of Muon Optimizer},
  author={Naoki Sato and Hiroki Naganuma and Hideaki Iiduka},
  year={2025},
  url={https://api.semanticscholar.org/CorpusID:280140878}
}

@article{an2025asgo,
  title={Asgo: Adaptive structured gradient optimization},
  author={An, Kang and Liu, Yuxing and Pan, Rui and Ren, Yi and Ma, Shiqian and Goldfarb, Donald and Zhang, Tong},
  journal={arXiv preprint arXiv:2503.20762},
  year={2025}
}

@article{kovalev2025understanding,
  title={Understanding gradient orthogonalization for deep learning via non-euclidean trust-region optimization},
  author={Kovalev, Dmitry},
  journal={arXiv preprint arXiv:2503.12645},
  year={2025}
}

@article{dozat2016incorporating,
  title={Incorporating nesterov momentum into adam},
  author={Dozat, Timothy},
  year={2016}
}

@inproceedings{SciQ,
  title={Crowdsourcing Multiple Choice Science Questions},
  author={Welbl, Johannes and Liu, Nelson F and Gardner, Matt},
  booktitle={Proceedings of the 3rd Workshop on Noisy User-generated Text},
  pages={94--106},
  year={2017}
}
\bibliographystyle{ims}


\newpage
\appendix




\section{Hyper-parameter Settings}\label{appendix_hyper} 
We list the architectural hyperparameters for GPT-2 models with 125M (small), 355M (medium), 770M (large) and 1.5B (XL) parameters in Table~\ref{arch_hyperparam}. Moreover, Table~\ref{tb:hyperparameter-train} display the training learning rates for models with different sizes.

\begin{table}[htb!]
\caption{Architecture hyperparameters for GPT-2 series models \citep{radford2019language}.}
\label{arch_hyperparam}
\begin{center}
\begin{tabular}{ccccc}
\toprule
\textbf{Model} & \textbf{\#Param} & \textbf{\#Layer} & \textbf{$n_{\text{head}}$} & \textbf{$d_{\text{emb}}$} \\
\midrule
GPT-2 small    & 125M             & 12               & 12                                  & 768                                 \\
GPT-2 medium   & 355M             & 24               & 16                                  & 1024                                \\
GPT-2 large    & 770M             & 36               & 20                                  & 1280                                \\
GPT-2 XL    & 1.5B             & 48               & 25                                  & 1600                                \\
\bottomrule
\end{tabular}

\end{center}

\end{table}


\begin{table}[htb!]
\caption{Learning rates for GPT-2 experiments for different datasets. }
\label{tb:hyperparameter-train}
\begin{center}

\begin{tabular}{cccc}
\hline
\textbf{Hyper-parameter}           & \textbf{GPT-2 Size} & OpenWebText      & FineWebEdu 100B        \\ \hline
\multirow{3}{*}{Max learning rate} & small (125M)        & 6e-3       & 1e-2        \\
& medium (355M)       & 5e-3       & 5e-3         \\
& large (770M)        & 5e-3       & 5e-3        \\
& XL (1.5B)        & -       & 3e-3        \\\hline
\multirow{3}{*}{Min learning rate} & small (125M)        & 3e-5       & 3e-5        \\
& medium (355M)       & 6e-5       & 6e-5       \\
& large (770M)        & 1e-5       & 1e-5        \\
& XL (1.5B)        & -       & 1e-5        \\\hline
\end{tabular}
\end{center}
\end{table}

\section{Additional Experiments}
\label{sec:additional_exp}
We just display the curves for training and validation losses for small models on the FineWeb-Edu 100B in Figure~\ref{fig:small_fw}. And we also display the curves for training and validation losses for the entire training process on the OpenWebText and FineWeb-Edu 100B datasets in Figures~\ref{fig:train_open_global}-\ref{fig:val_open_global} and Figures~\ref{fig:train_fw_global}-\ref{fig:val_fw_global}, respectively. And the 0-shot and 2-shot evaluation results for these models on two datasets are shown in Tables~\ref{tab:small-0-shot-fw}-\ref{tab:xl-0-shot-fw} and Tables~\ref{tab:small-0-shot-open}-\ref{tab:large-2-shot-open}, respectively. The results of AdamW are taken from~\citet{yuan2024mars}. According to these figures and tables, it can be observed that MARS-M can achieve better performances than AdamW and Moonlight optimizers in most of these cases.

\begin{figure}[htb!]
    \centering
    \includegraphics[width=0.45\linewidth]{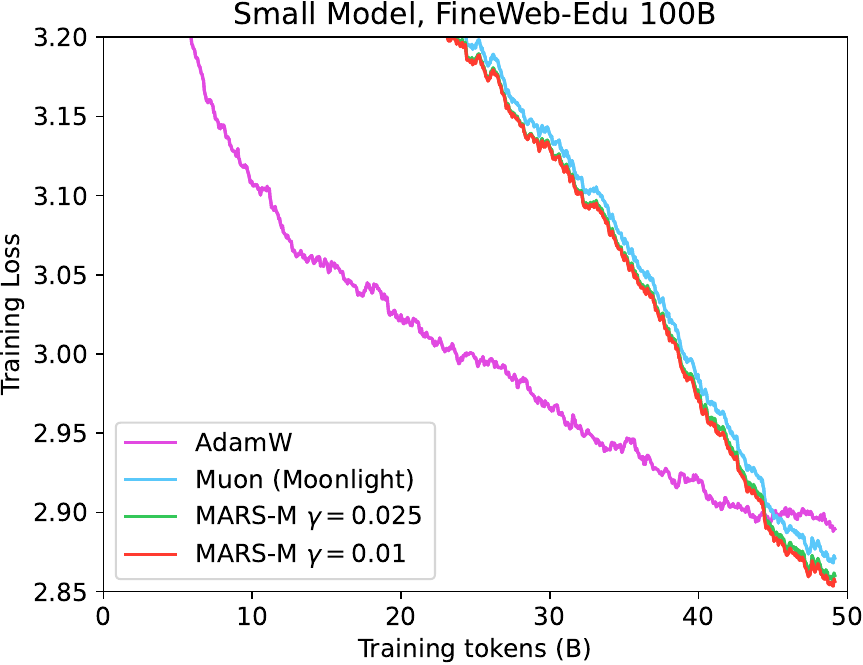}
    \includegraphics[width=0.45\linewidth]{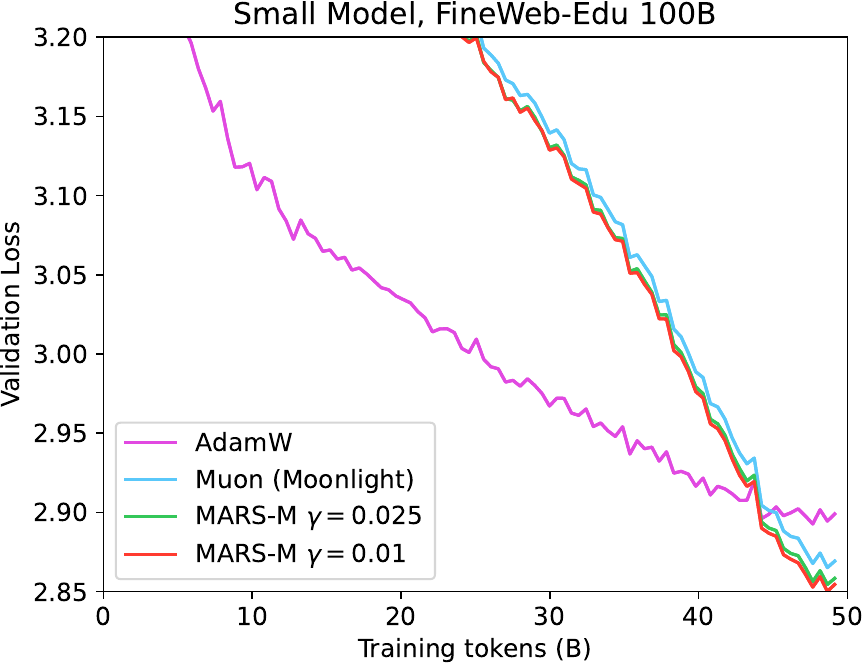}
    \caption{The training and validation loss of small-size models (125M) trained with different optimizers on the FineWeb-Edu 100B dataset.}
    \label{fig:small_fw}
\end{figure}

\begin{figure*}[htb!]
    \centering
    \includegraphics[width=0.32\linewidth]{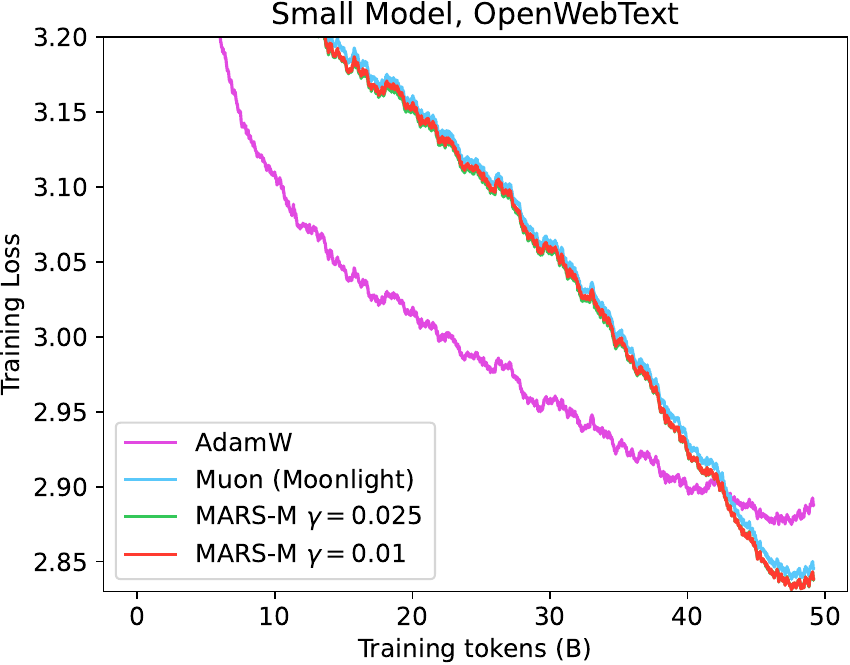}
    \includegraphics[width=0.32\linewidth]{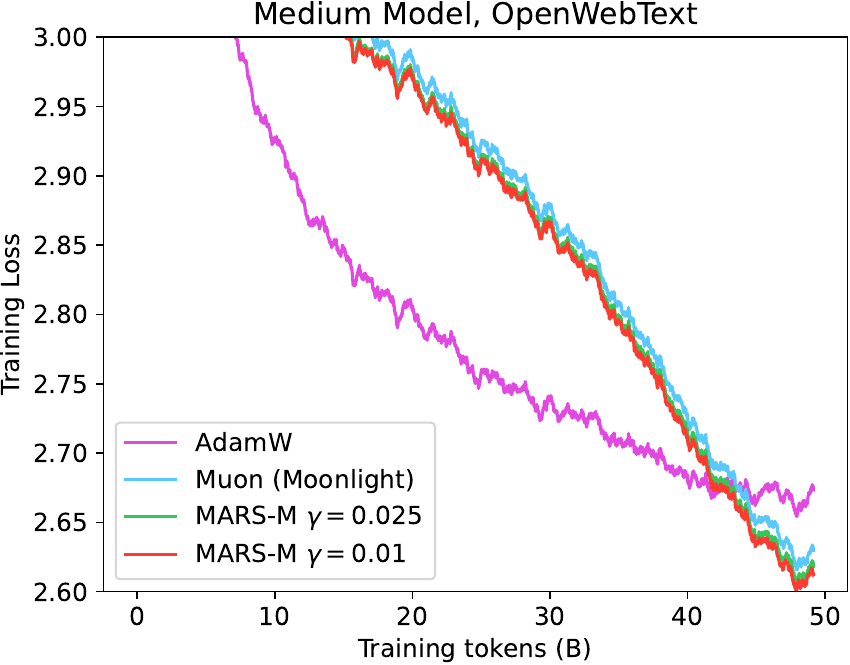}
    \includegraphics[width=0.32\linewidth]{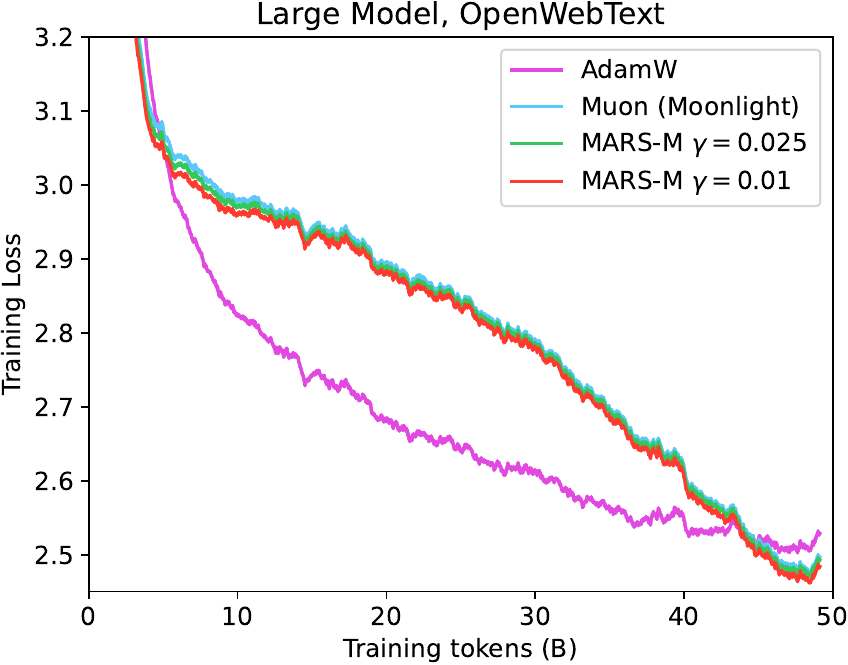}
    \caption{The zoomed-in training loss of small-size (125M), medium-size (355M) and large-size (770M) models trained with different optimizers on the OpenWebText dataset.}
    \label{fig:train_open_global}
\end{figure*}
\begin{figure*}[htb!]
    \centering
    \includegraphics[width=0.32\linewidth]{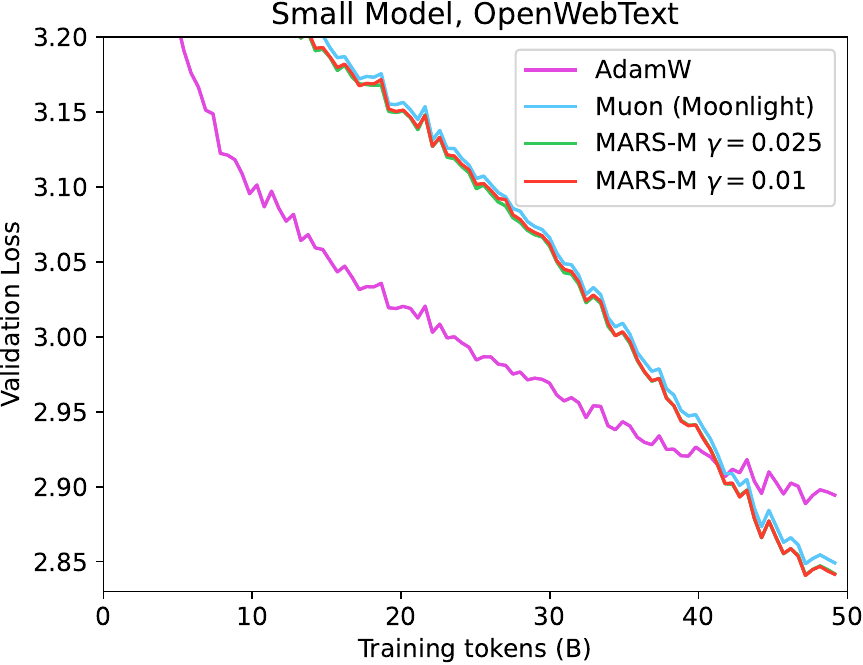}
    \includegraphics[width=0.32\linewidth]{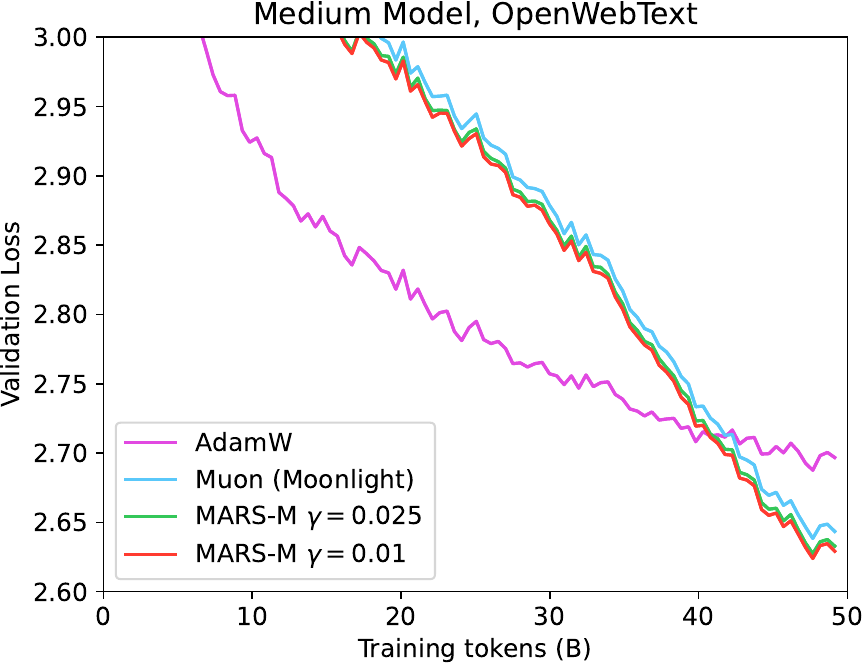}
    \includegraphics[width=0.32\linewidth]{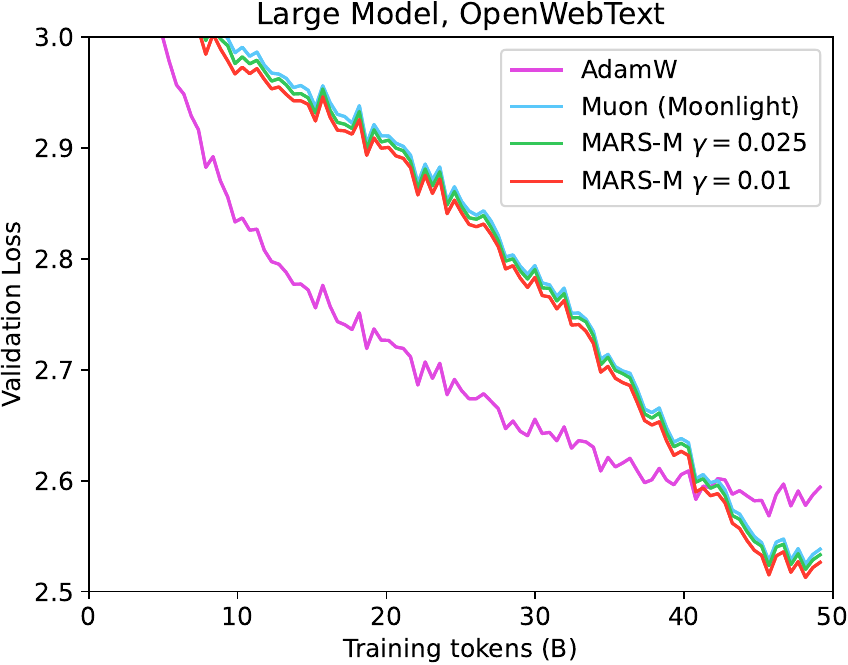}
    \caption{The zoomed-in training and validation loss of large-size models (770M) trained with different optimizers on the OpenWebText dataset.}
    \label{fig:val_open_global}
\end{figure*}
\begin{figure*}[htb!]
    \centering
    \includegraphics[width=0.32\linewidth]{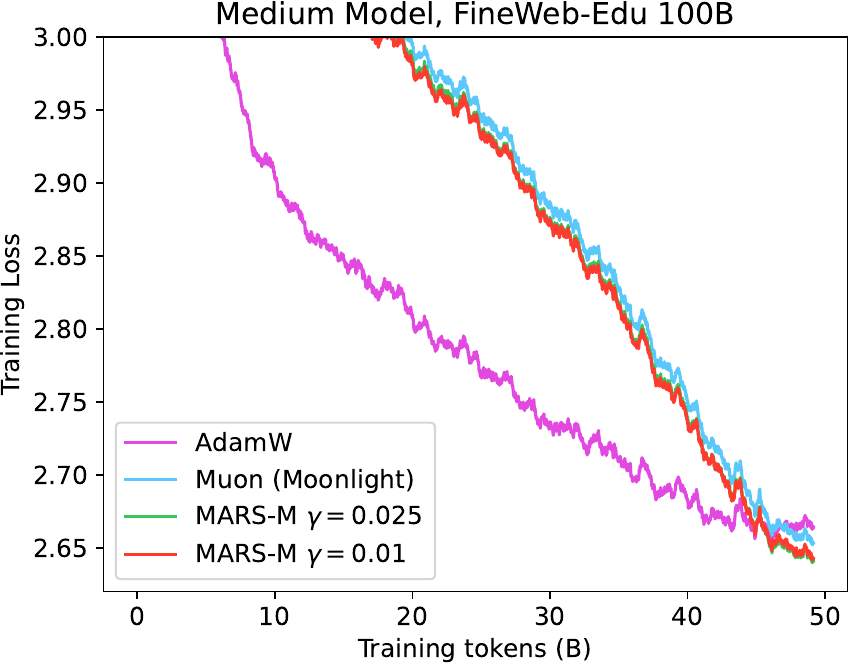}
    \includegraphics[width=0.32\linewidth]{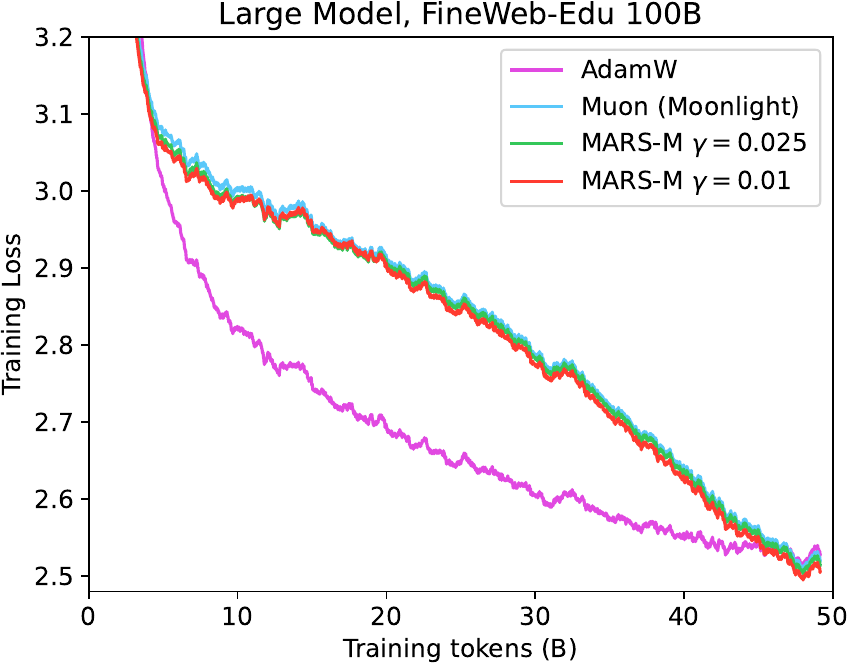}
    \includegraphics[width=0.32\linewidth]{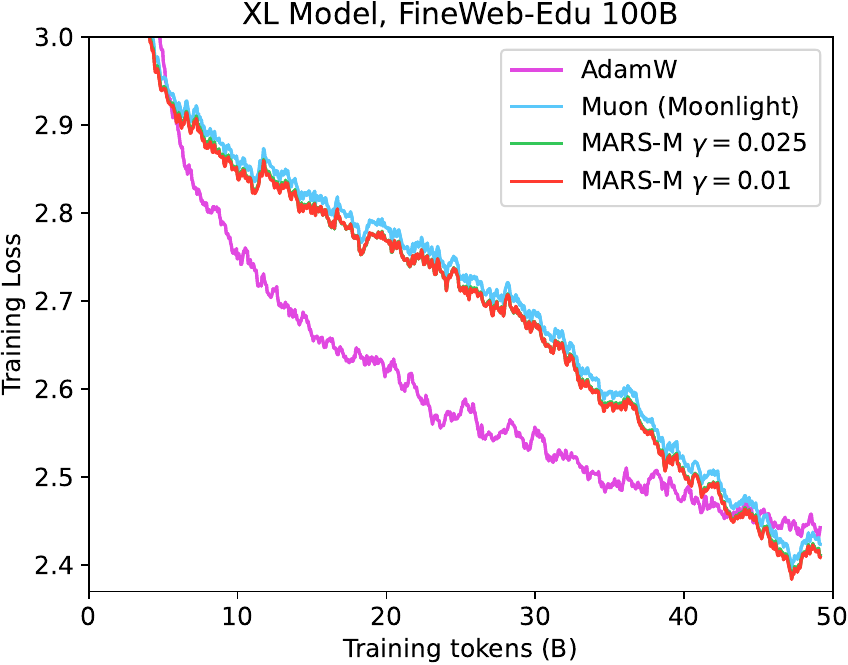}
    \caption{The zoomed-in training loss of medium-size (355M), large-size (770M) and XL-size (1.5B) models trained with different optimizers on the FineWeb-Edu 100B dataset.}
    \label{fig:train_fw_global}
\end{figure*}
\begin{figure*}[htb!]
    \centering
    \includegraphics[width=0.32\linewidth]{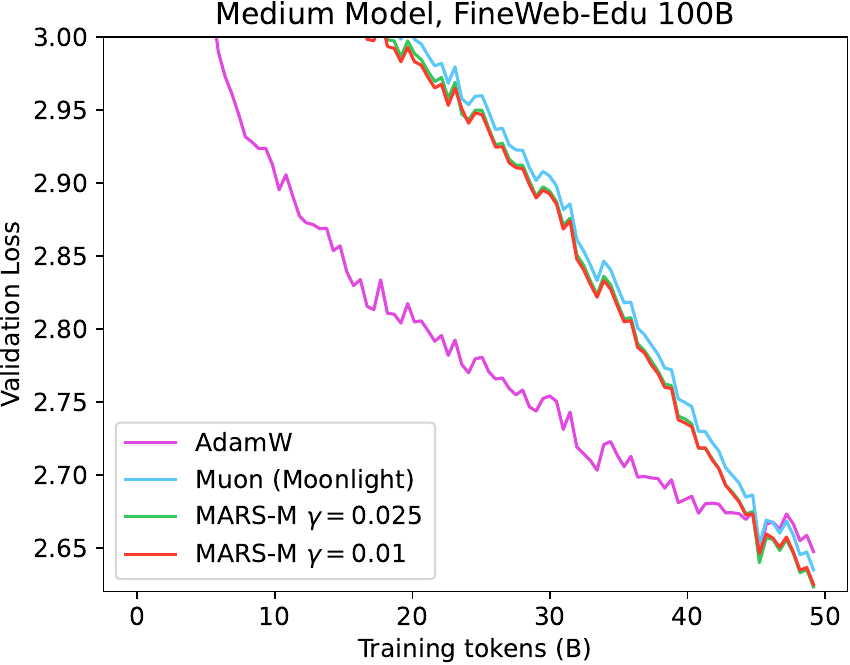}
    \includegraphics[width=0.32\linewidth]{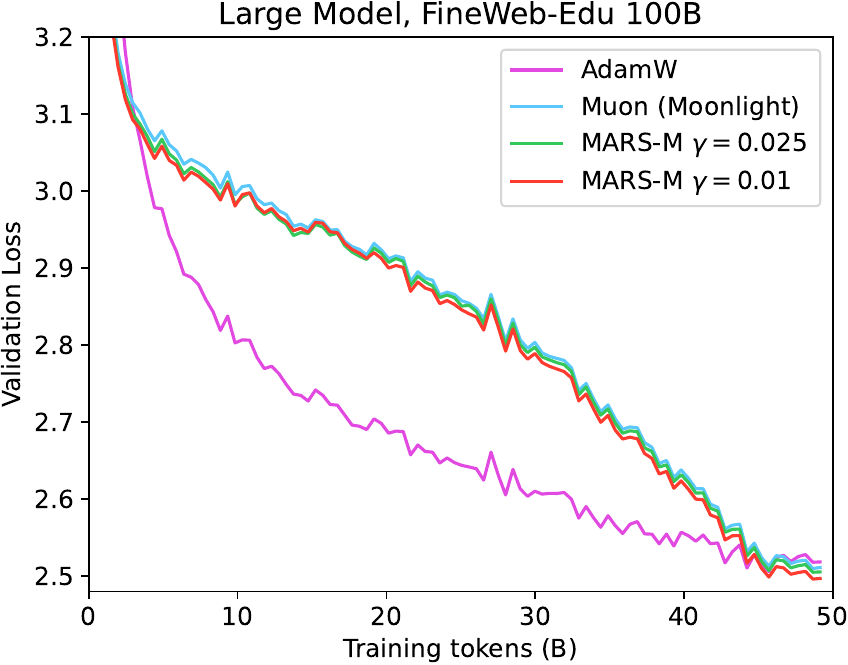}
    \includegraphics[width=0.32\linewidth]{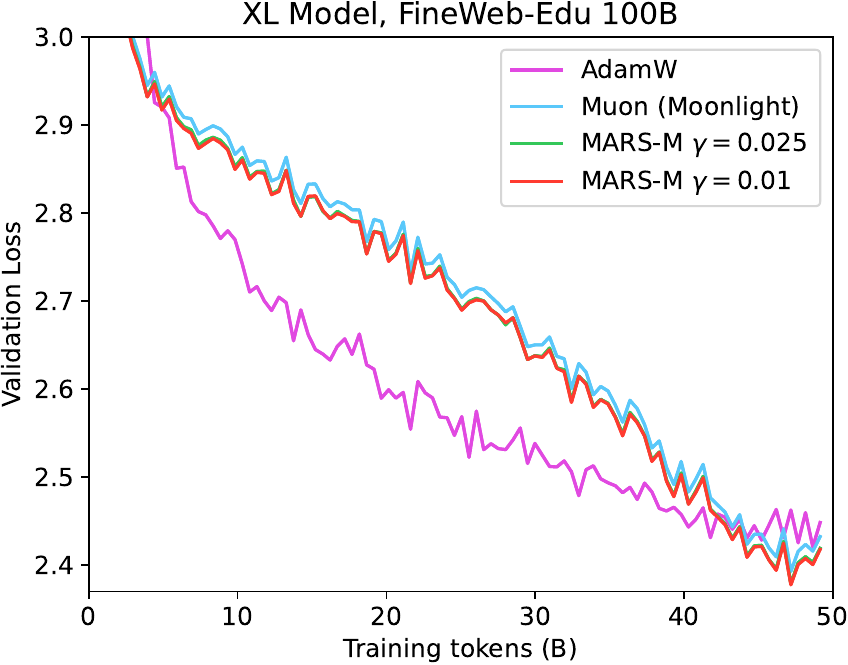}
    \caption{The zoomed-in validation loss of medium-size (355M), large-size (770M) and XL-size (1.5B) models trained with different optimizers on the FineWeb-Edu 100B dataset.}
    \label{fig:val_fw_global}
\end{figure*}

\begin{table}[htb!]
\centering
\caption{The evaluation results of small models pre-trained using the FineWeb-Edu 100B dataset (0-shot with lm-evaluation-harness). The best scores in each column are bolded. Abbreviations: WG = WinoGrande.}
\label{tab:small-0-shot-fw}
\resizebox{\textwidth}{!}{%
\begin{tabular}{cccccccccc}
\hline
Method & ARC-C & ARC-E & Hellaswag & MMLU & OpenBookQA & PIQA & SciQ & WG & Avg. \\ \hline
AdamW & 26.54 & 51.43 & 36.26 & 24.49 & 30.60 & 64.53 & 71.50 & 50.36 & 44.46 \\
Muon (Moonlight) & 26.88 & 52.48 & 37.58 & 23.12 & \textbf{32.80} & 64.85 & 71.20 & 50.67 & 44.95 \\
MARS-M ($\gamma=0.01$) & 26.28 & 51.47 & \textbf{37.77} & 23.26 & 31.80 & \textbf{65.94} & 70.90 & \textbf{52.01} & 44.93 \\
MARS-M ($\gamma=0.025$) & \textbf{27.82} & \textbf{52.57} & 37.74 & \textbf{26.01} & 32.00 & 65.23 & \textbf{71.90} & 51.14 & \textbf{45.55} \\ \hline
\end{tabular}%
}
\end{table}

\begin{table}[htb!]
\centering
\caption{The evaluation results of small models pre-trained using the FineWeb-Edu 100B dataset (2-shot with lm-evaluation-harness). The best scores in each column are bolded. Abbreviations: WG = WinoGrande.}
\label{tab:small-2-shot-fw}
\resizebox{\textwidth}{!}{%
\begin{tabular}{cccccccccc}
\hline
Method & ARC-C & ARC-E & Hellaswag & MMLU & OpenBookQA & PIQA & SciQ & WG & Avg. \\ \hline
Muon (Moonlight) & \textbf{28.41} & 57.62 & 37.00 & 25.83 & 30.00 & 64.20 & 82.90 & 51.38 & 47.17 \\
MARS-M ($\gamma=0.01$) & 27.82 & 58.16 & \textbf{37.50} & 25.66 & \textbf{31.60} & 65.02 & \textbf{86.40} & \textbf{53.51} & \textbf{48.21} \\
MARS-M ($\gamma=0.025$) & \textbf{28.41} & \textbf{58.42} & 37.06 & \textbf{26.46} & 31.20 & \textbf{65.34} & 82.10 & 50.75 & 47.47 \\ \hline
\end{tabular}%
}
\end{table}

\begin{table}[htb!]
\centering
\caption{The evaluation results of medium models pre-trained using the FineWeb-Edu 100B dataset (0-shot with lm-evaluation-harness). The best scores in each column are bolded. Abbreviations: WG = WinoGrande.}
\label{tab:medium-0-shot-fw}
\resizebox{\textwidth}{!}{%
\begin{tabular}{cccccccccc}
\hline
Method & ARC-C & ARC-E & Hellaswag & MMLU & OpenBookQA & PIQA & SciQ & WG & Avg. \\ \hline
AdamW & \textbf{32.68} & 59.34 & 45.05 & 24.98 & \textbf{35.40} & 68.88 & 76.10 & 54.85 & 49.66\\
Muon (Moonlight) & 32.17 & 58.92 & 45.46 & 23.28 & 34.80 & 69.64 & 76.80 & 54.14 & 49.40 \\
MARS-M ($\gamma=0.01$) & 30.55 & 59.60 & \textbf{46.11} & \textbf{25.13} & 31.80 & \textbf{69.91} & 77.50 & 54.06 & 49.33 \\
MARS-M ($\gamma=0.025$) & 32.25 & \textbf{60.65} & 45.92 & 24.52 & 33.80 & 67.90 & \textbf{78.30} & \textbf{56.20} & \textbf{49.94} \\ \hline
\end{tabular}%
}
\end{table}

\begin{table}[htb!]
\centering
\caption{The evaluation results of large models pre-trained using the FineWeb-Edu 100B dataset (0-shot with lm-evaluation-harness). The best scores in each column are bolded. Abbreviations: WG = WinoGrande.}
\label{tab:large-0-shot-fw}
\resizebox{\linewidth}{!}{
\begin{tabular}{cccccccccc}
\hline
Method                & ARC-C          & ARC-E          & Hellaswag      & MMLU           & OpenBookQA     & PIQA           & SciQ           & WG             & Avg.            \\ \hline
AdamW & 34.64 & 63.22 & 50.10 & 25.45 & 38.60 & 70.89 & 80.60 & 57.54 & 52.63\\
Muon (Moonlight)      & 34.30          & 63.34          & 51.93          & 24.18          & \textbf{38.80} & \textbf{72.31} & 81.50          & \textbf{57.77} & 53.07          \\
MARS-M ($\gamma=0.01$)  & \textbf{36.18} & \textbf{64.52} & \textbf{52.29} & \textbf{25.89} & 37.60          & 71.65          & \textbf{82.30} & 56.59          & \textbf{53.45} \\
MARS-M ($\gamma=0.025$) & 31.74          & 62.84          & 51.91          & 24.35          & 35.80          & 70.35          & 79.70          & 54.85          & 51.45          \\ \hline
\end{tabular}
}
\end{table}

\begin{table}[htb!]
\centering
\caption{The evaluation results of XL models pre-trained using the FineWeb-Edu 100B dataset (0-shot with lm-evaluation-harness). The best scores in each column are bolded. Abbreviations: WG = WinoGrande.}
\label{tab:xl-0-shot-fw}
\resizebox{\textwidth}{!}{%
\begin{tabular}{cccccccccc}
\hline
Method & ARC-C & ARC-E & Hellaswag & MMLU & OpenBookQA & PIQA & SciQ & WG & Avg. \\ \hline
AdamW & 38.40 & 68.22 & 53.93 & 25.47 & 39.00 & 72.69 & \textbf{85.30} & 54.78 & 54.72 \\
Muon (Moonlight) & \textbf{39.08} & 67.30 & 56.37 & \textbf{25.32} & \textbf{43.20} & \textbf{73.78} & 83.50 & \textbf{58.41} & \textbf{55.87} \\
MARS-M ($\gamma=0.01$) & 38.74 & 68.14 & 56.83 & 24.68 & 39.80 & 73.23 & 84.70 & 57.22 & 55.42 \\
MARS-M ($\gamma=0.025$) & 38.31 & \textbf{68.56} & \textbf{57.12} & 23.57 & 41.40 & 73.72 & \textbf{85.30} & 56.75 & 55.59 \\ \hline
\end{tabular}%
}
\end{table}

\begin{table}[htb!]
\centering
\caption{The evaluation results of small models pre-trained using the OpenWebText dataset (0-shot with lm-evaluation-harness). The best scores in each column are bolded. Abbreviations: WG = WinoGrande.}
\label{tab:small-0-shot-open}
\resizebox{\textwidth}{!}{%
\begin{tabular}{cccccccccc}
\hline
Method & ARC-C & ARC-E & Hellaswag & MMLU & OpenBookQA & PIQA & SciQ & WG & Avg. \\ \hline
AdamW & 22.27 & 41.37 & 31.73 & \textbf{22.97} & 27.80 & 63.00 & \textbf{67.50} & \textbf{52.01} & 41.08\\
Muon (Moonlight) & \textbf{23.81} & \textbf{41.67} & 33.19 & 22.83 & 28.60 & 63.06 & 66.50 & 51.38 & 41.38 \\
MARS-M ($\gamma=0.01$) & 23.72 & 40.66 & \textbf{33.36} & \textbf{22.97} & 27.80 & \textbf{74.80} & 65.10 & 51.38 & \textbf{42.47} \\
MARS-M ($\gamma=0.025$) & 23.04 & 40.49 & 33.44 & 22.93 & \textbf{30.20} & 63.22 & 66.40 & 50.83 & 41.32 \\ \hline
\end{tabular}%
}
\end{table}

\begin{table}[htb!]
\centering
\caption{The evaluation results of small models pre-trained using the OpenWebText dataset (2-shot with lm-evaluation-harness). The best scores in each column are bolded. Abbreviations: WG = WinoGrande.}
\label{tab:small-2-shot-open}
\resizebox{\textwidth}{!}{%
\begin{tabular}{cccccccccc}
\hline
Method & \multicolumn{1}{c}{ARC-C} & \multicolumn{1}{c}{ARC-E} & \multicolumn{1}{c}{Hellaswag} & \multicolumn{1}{c}{MMLU} & \multicolumn{1}{c}{OpenBookQA} & \multicolumn{1}{c}{PIQA} & \multicolumn{1}{c}{SciQ} & \multicolumn{1}{c}{WG} & \multicolumn{1}{c}{Avg.} \\ \hline
Muon (Moonlight) & 24.23 & 43.18 & 33.20 & 25.00 & 26.80 & \textbf{62.79} & 75.20 & \textbf{53.91} & \textbf{43.04} \\
MARS-M ($\gamma=0.01$) & \textbf{24.49} & \textbf{43.52} & \textbf{33.47} & \textbf{25.35} & 26.20 & 62.02 & 73.50 & 51.46 & 42.50 \\
MARS-M ($\gamma=0.025$) & 23.38 & 43.43 & 33.39 & 23.99 & \textbf{27.40} & 62.40 & \textbf{77.40} & 49.41 & 42.60 \\ \hline
\end{tabular}%
}
\end{table}

\begin{table}[htb!]
\centering
\caption{The evaluation results of medium models pre-trained using the OpenWebText dataset (0-shot with lm-evaluation-harness). The best scores in each column are bolded. Abbreviations: WG = WinoGrande.}
\label{tab:medium-0-shot-open}
\resizebox{\textwidth}{!}{%
\begin{tabular}{cccccccccc}
\hline
Method & ARC-C & ARC-E & Hellaswag & MMLU & OpenBookQA & PIQA & SciQ & WG & Avg. \\ \hline
AdamW & 23.98 & 43.43 & 37.76 & 22.80 & 27.20 & 65.56 & 67.60 & 52.49 & 42.60 \\
Muon (Moonlight) & \textbf{25.60} & 44.99 & 40.00 & 23.20 & 31.20 & 65.94 & \textbf{72.00} & 52.25 & 44.40 \\
MARS-M ($\gamma=0.01$) & 24.49 & 44.49 & \textbf{40.22} & \textbf{23.34} & \textbf{33.40} & \textbf{67.08} & 68.01 & \textbf{54.46} & \textbf{44.44} \\
MARS-M ($\gamma=0.025$) & 24.57 & \textbf{46.42} & 40.21 & 23.28 & 32.00 & 66.27 & 70.00 & 52.64 & 44.42 \\ \hline
\end{tabular}%
}
\end{table}

\begin{table}[htb!]
\centering
\caption{The evaluation results of medium models pre-trained using the OpenWebText dataset (2-shot with lm-evaluation-harness). The best scores in each column are bolded. Abbreviations: WG = WinoGrande.}
\label{tab:medium-2-shot-open}
\resizebox{\textwidth}{!}{%
\begin{tabular}{cccccccccc}
\hline
Method & ARC-C & ARC-E & Hellaswag & MMLU & OpenBookQA & PIQA & SciQ & WG & Avg. \\ \hline
Muon (Moonlight) & \textbf{27.13} & 48.82 & 39.77 & 24.93 & 31.40 & 67.03 & 81.30 & 52.80 & 46.65 \\
MARS-M ($\gamma=0.01$) & 26.79 & \textbf{50.17} & \textbf{40.87} & \textbf{25.42} & 28.80 & \textbf{67.14} & 81.60 & 52.64 & 46.68 \\
MARS-M ($\gamma=0.025$) & 26.37 & 48.82 & 40.53 & 24.68 & \textbf{31.60} & 65.78 & \textbf{82.40} & \textbf{53.67} & \textbf{46.73} \\ \hline
\end{tabular}%
}
\end{table}

\begin{table}[htb!]
\centering
\caption{The evaluation results of large models pre-trained using the OpenWebText dataset (0-shot with lm-evaluation-harness). The best scores in each column are bolded. Abbreviations: WG = WinoGrande.}
\label{tab:large-0-shot-open}
\resizebox{\textwidth}{!}{%
\begin{tabular}{cccccccccc}
\hline
Method & ARC-C & ARC-E & Hellaswag & MMLU & OpenBookQA & PIQA & SciQ & WG & Avg. \\ \hline
AdamW & 26.19 & 46.30 & 41.70 & 23.10 & 31.40 & 68.12 & 72.80 & 51.46 & 45.13\\
Muon (Moonlight) & \textbf{27.30} & 47.85 & 45.56 & 23.85 & 30.80 & 68.93 & 72.50 & 55.64 & 46.55 \\
MARS-M ($\gamma=0.01$) & 26.88 & \textbf{49.37} & \textbf{45.74} & \textbf{24.41} & 32.00 & \textbf{69.37} & \textbf{74.30} & 53.43 & 46.94 \\
MARS-M ($\gamma=0.025$) & 27.05 & 48.48 & 45.19 & 23.59 & \textbf{33.20} & 68.82 & 72.50 & \textbf{57.30} & \textbf{47.02} \\ \hline
\end{tabular}%
}
\end{table}

\begin{table}[htb!]
\centering
\caption{The evaluation results of large models pre-trained using the OpenWebText dataset (2-shot with lm-evaluation-harness). The best scores in each column are bolded. Abbreviations: WG = WinoGrande.}
\label{tab:large-2-shot-open}
\resizebox{\textwidth}{!}{%
\begin{tabular}{cccccccccc}
\hline
Method & ARC-C & ARC-E & Hellaswag & MMLU & OpenBookQA & PIQA & SciQ & WG & Avg. \\ \hline
Muon (Moonlight) & 27.90 & 53.45 & 45.27 & \textbf{25.77} & \textbf{32.80} & \textbf{69.53} & 84.20 & 56.12 & 49.38 \\
MARS-M ($\gamma=0.01$) & \textbf{27.99} & \textbf{53.96} & \textbf{45.94} & 24.76 & 31.40 & 68.82 & \textbf{86.40} & \textbf{56.27} & \textbf{49.44} \\
MARS-M ($\gamma=0.025$) & 27.47 & 53.07 & 45.33 & 25.10 & 31.80 & 69.21 & 85.50 & 55.49 & 49.12 \\ \hline
\end{tabular}%
}
\end{table}

\newpage
\section{Proof of Theorem~\ref{thm:mars-moonlight-convergence}}
\label{sec:proof_of_thm}
For terms for variance reduction, we have the following lemma from~\citet{yuan2024mars}:
\begin{lemma}[\citealt{yuan2024mars}]\label{lem:recursion} 
In Algorithm~\ref{alg:MARS_moonlight}. Under Assumption \ref{assum:variance} and \ref{assum:L-smooth}, if $1 \geq \beta_{t+1} \geq 0$, $\forall t$, denote $\bDelta_t := \nabla f(\Xb_{t+1}, \bxi_{t+1}) - \nabla f(\Xb_t, \bxi_{t+1})$, under approximate choice of $\gamma_{t+1}$: 
\begin{align*}
    \gamma_{t+1} = 1 - \frac{\Big(G_{t+1} + \beta_{t+1} \big(\EE \|\bDelta_t\|_F^2 - \|\EE \bDelta_t\|_F^2\big) \Big)}{\beta_{t+1} \EE \|\bDelta_t\|_F^2}
        =
    \frac{\beta_{t+1}\|\EE \bDelta_t\|_F^2 - G_{t+1}}{\beta_{t+1} \EE \|\bDelta_t\|_F^2},
\end{align*}
we have
	\begin{align*}
		 \EE\|\nabla F(\Xb_{t+1})-\Mb_{t+1}\|_F^2
		 	&\leq
		 \beta_{t+1}^2\EE\|\nabla F(\Xb_t)-\Mb_t\|_F^2+2\beta_{t+1}^2L^2\EE\|\Xb_{t+1}-\Xb_t\|_F^2+2(1-\beta_{t+1})^2\sigma^2
		 	-
	\tilde{M}_{t+1},
	\end{align*}
	where 
\begin{align}
    \tilde{M}_{t+1} 
        :=
    \EE \|\bDelta_t\|_F^2 \Bigg(A_{t+1}^2
        -
    \Big(
        \beta_{t+1} (1 - \gamma_{t+1})
            -
        A_{t+1}
    \Big)^2\Bigg)\label{eq:M_value}
    ,
\end{align}
    \begin{align*}
        A_{t+1} := \frac{G_{t+1} + \beta_{t+1} \text{tr}\big(\Var\big( \bDelta_t \big)\big)}{\EE \| \bDelta_t\|_F^2},
    \end{align*},
and
	\begin{align*}
		G_{t+1} := 
		(1 - \beta_{t+1})\EE\bigg\langle \bDelta_t, \nabla f(\Xb_{t+1}, \bxi_{t+1}) - \nabla F(\Xb_{t+1})\bigg\rangle+
	\beta_{t+1} \EE 
	\bigg\langle 
	\bDelta_t, \Mb_t-\nabla F(\Xb_t) \bigg\rangle.
	\end{align*}
\end{lemma}

We also need the following lemmas, which can be seen as an extension of Lemma C.3 in \citet{yuan2024mars}.
\begin{lemma}
\label{lemma:diff_F_X}
For Algorithm \ref{alg:MARS_moonlight}, under Assumption \ref{assum:L-smooth}, we have the following inequality hold:
$$
F(\mathbf{X}_{t+1}) \le F(\mathbf{X}_{t}) - \eta_t\|\mathbf{M}_t\|_F + \frac{\eta_t}{\rho}\|\nabla F(\mathbf{X}_{t}) - \mathbf{M}_t\|_F^2 + \frac{\eta_t \rho}{4}\|\mathbf{O}_t\|_F^2 + \frac{L\eta_t^2}{2}\|\mathbf{O}_t\|_F^2,
$$ 
where $\rho > 0$ is a an arbitrarily chosen parameter.
\end{lemma}

\begin{proof}
According to Assumption \ref{assum:L-smooth}, we have the upper bound for the function value:
$$
\begin{aligned}
F(\mathbf{X}_{t+1}) &\le F(\mathbf{X}_{t}) + \langle \nabla F(\mathbf{X}_{t}), \mathbf{X}_{t+1}-\mathbf{X}_{t} \rangle + \frac{L}{2}\|\mathbf{X}_{t+1}-\mathbf{X}_{t}\|_F^2\\
& = F(\mathbf{X}_{t}) +  \langle \nabla F(\mathbf{X}_{t}), \mathbf{X}_{t+1}-\mathbf{X}_{t} \rangle + \frac{L}{2}\|\mathbf{X}_{t+1}-\mathbf{X}_{t}\|_F^2\\
&= F(\mathbf{X}_{t}) - \eta_t \langle \mathbf{M}_t, \Ob_t \rangle - \eta_t \langle \nabla F(\mathbf{X}_{t}) - \mathbf{M}_t, \Ob_t \rangle + \frac{L\eta_t^2}{2}\|\mathbf{O}_t\|_F^2
\end{aligned}
$$
Here $\mathbf{O}_t=\Ub_t\Vb_t^\top$ is defined in~\eqref{eq:SVD}. \footnote{In Algorithm~\ref{alg:MARS_moonlight}, $\mathbf{O}_t$ is computed via a (truncated) Newton--Schulz iteration; for the purpose of this analysis, we treat $\mathbf{O}_t$ as the exact polar factor $\Ub_t\Vb_t^\top$ and ignore the approximation error.}According to the definition of $\mathbf{O}_t$, we have:
\begin{align*}
-\eta_t \langle \mathbf{M}_t, \mathbf{O}_t \rangle = -\eta_t \langle \mathbf{M}_t, \mathbf{U}_r \mathbf{V}_r^\top \rangle = -\eta_t \|\mathbf{M}_t\|_* \le -\eta_t \|\mathbf{M}_t\|_F, 
\end{align*}
where the last inequality follows from  $\|\mathbf{M}_t\|_* \ge \|\mathbf{M}_t\|_F$.
Moreover, by AM-GM inequality, it holds that 
\[
-\eta_t \langle \nabla F(\mathbf{X}_{t}) - \mathbf{M}_t, \mathbf{O}_t \rangle 
\le \eta_t \left( \frac{1}{\rho}\|\nabla F(\mathbf{X}_{t})-\mathbf{M}_t\|_F^2 + \frac{\rho}{4}\|\mathbf{O}_t\|_F^2 \right)
\]
for some constant $\rho>0$. Putting all pieces together, we can obtain
\[
F(\mathbf{X}_{t+1}) \le F(\mathbf{X}_{t}) - \eta_t\|\mathbf{M}_t\|_F + \frac{\eta_t}{\rho}\|\nabla F(\mathbf{X}_{t}) - \mathbf{M}_t\|_F^2 + \frac{\eta_t \rho}{4}\|\mathbf{O}_t\|_F^2 + \frac{L\eta_t^2}{2}\|\mathbf{O}_t\|_F^2
\]
This completes the proof.
\end{proof}

\begin{lemma}\label{lem:eta_diff} Let $\eta_t=(s+t)^{-2/3},s\ge 1$, $\forall t\ge 1$. Then $\eta_t^{-1}-\eta_{t-1}^{-1}\le \eta_{t}^{1/2}$, $\forall t\ge 1$.
\end{lemma}
\begin{proof}
By the definition of $\eta_t$, it holds that
\begin{align*}
    \frac{1}{\eta_t}-\frac{1}{\eta_{t-1}}=(s+t)^{2/3}-(s+t-1)^{2/3}\le\frac{2}{3(s+t-1)^{1/3}}\le\eta_{t}^{1/2},
\end{align*}
where the first inequality follows by the concavity of $h(\xb)=x^{2/3}$. This finishes the proof.
\end{proof}

Now we're ready to prove Theorem~\ref{thm:mars-moonlight-convergence}.
\begin{proof}[Proof of Theorem~\ref{thm:mars-moonlight-convergence}]\footnote{Here we just ignore the factor $0.2\cdot\sqrt{\max(m,n)}$ since such a factor can be integrated into $\eta_t$.}
First, we define the Lyapunov function as
\begin{align*}
    \Phi_t=\EE\Big[F(\Xb_t)+\frac{\rho_t}{16L^2\eta_{t-1}}\cdot\|\nabla F(\Xb_t)-\Mb_t\|_F^2\Big], \quad \forall t\ge 1.
\end{align*}
where $\rho_t=4\sqrt{2}L\sqrt{\eta_t}$. Then we calculate the difference between two consecutive Lyapunov functions as:
\begin{align}\label{eq:decomp_main}
    \Phi_{t+1}-\Phi_t&=
    \underbrace{\EE[F(\Xb_{t+1})-F(\Xb_t)]}_{\mbox{$I_1$}}\nonumber\\
        &\qquad+
    \underbrace{\EE\Bigg[\frac{\rho_{t+1}}{16L^2\eta_t}\cdot\|\nabla F(\Xb_{t+1})-\Mb_{t+1}\|_F^2-\frac{\rho_t}{16L^2\eta_{t-1}}\cdot\|\nabla F(\Xb_t)-\Mb_t\|_F^2\Bigg]}_{\mbox{$I_2$}}.
\end{align}
For $I_1$, we use Lemma~\ref{lemma:diff_F_X} to obtain
\begin{align}\label{eq:I1_resu_main}
    I_1\le\EE\Big[- \eta_t\|\mathbf{M}_t\|_F+\frac{\eta_t}{\rho_t}\cdot \|\nabla F(\Xb_t)-\Mb_t\|_F^2 + \frac{\eta_t \rho_t}{4}\|\mathbf{O}_t\|_F^2 + \frac{L\eta_t^2}{2}\|\mathbf{O}_t\|_F^2\Big].
\end{align}
For $I_2$, we use Lemma \ref{lem:recursion} to obtain
\begin{align}
    I_2&=\EE\Big[\frac{\rho_{t+1}}{16L^2\eta_t}\cdot\|\nabla F(\Xb_{t+1})-\Mb_{t+1}\|_F^2-\frac{\rho_t}{16L^2\eta_{t-1}}\cdot\|\nabla F(\Xb_t)-\Mb_t\|_F^2\Big]\notag 
        \\&\le 
    \frac{\rho_t}{16L^2}\cdot \Big(\frac{\beta_{t+1}^2}{\eta_t}-\frac{1}{\eta_{t-1}}\Big)\EE\|\nabla F(\Xb_t)-\Mb_t\|_F^2+\frac{\rho_t\beta_{t+1}^2}{8\eta_t}\cdot\EE\|\Xb_{t+1}-\Xb_t\|_F^2+\frac{\rho_t(1-\beta_{t+1})^2\sigma^2}{8L^2\eta_t}
       \nonumber\\
       &\qquad-
    \frac{\rho_t}{16L^2 \eta_t} \tilde{M}_{t+1}
    \notag
    \\
    &\le \frac{\rho_t}{16L^2}\cdot \Big(\frac{\beta_{t+1}^2}{\eta_t}-\frac{1}{\eta_{t-1}}\Big)\EE\|\nabla F(\Xb_t)-\Mb_t\|_F^2+\frac{\rho_t}{8\eta_t}\cdot\EE\|\Xb_{t+1}-\Xb_t\|_F^2+\frac{\rho_t \eta_t\sigma^2}{2L^2}
        -
    \frac{\rho_t}{16L^2 \eta_t} \tilde{M}_{t+1}
    \label{eq:I_2_main}
    , 
\end{align}
where the first inequality is due to $\rho_{t+1}\le\rho_t$, and the last inequality follows from the definition that $\beta_{ t+1} = 1 - 2\eta_t$.
Further, for the first term on the right hand side, we have
\begin{align*}
    \frac{\rho_t}{16L^2}\cdot \Big(\frac{\beta_{t+1}^2}{\eta_t}-\frac{1}{\eta_{t-1}}\Big)
        &\le
    \frac{\rho_t}{16L^2}\cdot \Big(\frac{\beta_{t+1}}{\eta_t}-\frac{1}{\eta_{t-1}}\Big)
        =
    \frac{\rho_t}{16L^2}\cdot \Big(\frac{1-2 \eta_t}{\eta_t}-\frac{1}{\eta_{t-1}}\Big)\\
    &
        = 
    \frac{\rho_t}{16L^2}\cdot \Big(\frac{1}{\eta_t}-\frac{1}{\eta_{t-1}}- 2\Big).
\end{align*}
From Lemma \ref{lem:eta_diff}, we know that $\frac{1}{\eta_t}-\frac{1}{\eta_{t-1}}\leq \eta_t^{1/2}$. Since $\rho_t^2= 32 L^2 \eta_t$, we obtain\footnote{As long as $\eta_t \leq 1/2$, $1 \geq \beta_{t+1} = 1 - 2\eta_t \ge 0$ is satisfied.}
\begin{align}
    \frac{\rho_t}{16L^2}\cdot \Big(\frac{\beta_{t+1}^2}{\eta_t}-\frac{1}{\eta_{t-1}}\Big)\le \frac{\rho_t}{16L^2}\cdot \Big(\eta_t^{1/2}- 2\Big)\le \frac{\rho_t}{16L^2}\cdot \Big(1 - 2\Big)  = - 2\eta_t\rho_t^{-1}.\label{eq:eta_bound_main}
\end{align}
Bringing \eqref{eq:eta_bound_main} into \eqref{eq:I_2_main}, we arrive at the upper bound for $I_2$:
\begin{align}\label{eq:I2_resu_main}
I_2\le -\frac{2\eta_t}{\rho_t}\EE\|\nabla F(\Xb_t)-\Mb_t\|_F^2+\frac{\rho_t}{8\eta_t}\cdot\EE\|\Xb_{t+1}-\Xb_t\|_F^2+\frac{\rho_t  \eta_t \sigma^2}{2L^2}
        -
    \frac{\rho_t}{16L^2 \eta_t} \tilde{M}_{t+1}
.
\end{align}
Now combining \eqref{eq:decomp_main}, \eqref{eq:I1_resu_main} and \eqref{eq:I2_resu_main}, we derive
\begin{align}
    \Phi_{t+1}-\Phi_t
    &\le 
    -\frac{\eta_t}{\rho_t}\EE\|\nabla F(\Xb_t)-\Mb_t\|_F^2-\eta_t\EE\|\Mb_t\|_F+\frac{\rho_t}{8\eta_t}\cdot\EE\|\Xb_{t+1}-\Xb_t\|_F^2
    +
    \frac{\rho_t \eta_t \sigma^2}{2L^2}
       \nonumber\\
    &\qquad  -
    \frac{\rho_t}{16L^2 \eta_t} \tilde{M}_{t+1} + \frac{\eta_t \rho_t}{4}\|\mathbf{O}_t\|_F^2+ \frac{L\eta_t^2}{2}\|\mathbf{O}_t\|_F^2\label{eq:diff_Phi}
    .
\end{align}
Moreover, according to the definition of $\Ob_t$, we have
\begin{align*}
\|\mathbf{O}_t\|_F^2\le\sum_{i=1}^n 1 = n.
\end{align*}
Therefore, 
\begin{align*}
\frac{\rho_t}{8\eta_t}\cdot\EE\|\Xb_{t+1}-\Xb_t\|_F^2&=\frac{\rho_t}{8\eta_t}\cdot\EE\|\eta_t\Ob_t\|_F^2=\frac{\rho_t\eta_t}{8}\cdot\EE\|\Ob_t\|_F^2\le \frac{\rho_t\eta_t n}{8},\\
\frac{\eta_t \rho_t}{4}\|\mathbf{O}_t\|_F^2&\le\frac{\rho_t\eta_t  n}{4},\\
\frac{L\eta_t^2}{2}\|\mathbf{O}_t\|_F^2&\le \frac{L\eta_t^2n}{2} .
\end{align*}
Taking a telescoping sum for $t=1,\cdots, T$ in~\eqref{eq:diff_Phi}, and applying $\rho_t=4\sqrt{2}L\sqrt{\eta_t}$ give
\begin{align*}
&\sum_{t=1}^{T}\Big(\frac{\sqrt{\eta_t}}{4\sqrt{2}L}\EE\|\nabla F(\Xb_t)-\Mb_t\|_F^2+\eta_t\EE\|\Mb_t\|_F\Big)\\
    & \leq \Phi_{1}-\Phi_{T+1}+\frac{\sigma^2}{2L^2}\sum_{t=1}^{T}\rho_t\eta_t+\sum_{t=1}^{T}\big(\frac{ 3\rho_t\eta_t n}{8}+\frac{L\eta_t^2n}{2}\big) -\sum_{t=1}^T\frac{\sqrt{2}}{4L\sqrt{\eta_t}} \tilde{M}_{t+1}\\
    &\le 
    \Phi_{1}-\Phi_{T+1}+\Big(\frac{2\sqrt{2} \sigma^2}{L}+\frac{3\sqrt{2}Ln}{2}\Big)\sum_{t=1}^{T}\frac{1}{s+t}+\frac{Ln}{2}\sum_{t=1}^T\frac{1}{(s+t)^{4/3}}-\sum_{t=1}^T
    \frac{\sqrt{2}}{4L\sqrt{\eta_t}} \tilde{M}_{t+1}
    \\
    &\le 
    \Phi_{1}-\Phi_{T+1}+B\cdot\log(s+T)
        +\frac{3Ln}{2s^{1/3}}-\sum_{t=1}^T\frac{\sqrt{2}}{4L\sqrt{\eta_t}}\tilde{M}_{t+1}
    ,
\end{align*}
where $B=\Big(\frac{2\sqrt{2} \sigma^2}{L}+\frac{3\sqrt{2}Ln}{2}\Big)$. By the definition of $\Phi_t$, we have $\Phi_{T+1}\ge F(\Xb_{T+1}) \ge \min_\Xb F(\Xb)$. And for $\Phi_1$,  
\begin{align*}
    \Phi_1&=\EE\Big[F(\Xb_1)+\frac{\rho_1 s^{2/3}}{16L^2}\cdot\|\nabla F(\Xb_1)-\Mb_1\|_F^2\Big]\\
        &\le
    F(\Xb_1)+\frac{\sqrt{2}s^{1/3}}{4L}\cdot\EE[\|\nabla F(\Xb_1)-\nabla f(\Xb_1, \bxi_1)\|_F^2]\\
        &\leq 
    F(\Xb_1)+\frac{\sqrt{2}s^{1/3}\sigma^2}{4L}
    .
\end{align*}
Consequently, defining $G=F(\Xb_1)-\min_\Xb F(\Xb)+\frac{\sqrt{2}s^{1/3}\sigma^2}{4L}+\frac{3Ln}{2s^{1/3}}$, the following inequality holds:
\begin{align*}
    \frac{1}{T}\sum_{t=1}^T\EE\|\nabla F(\Xb_t)-\Mb_t\|_F^2&\le \frac{4\sqrt{2}LG}{T\sqrt{\eta_T}}+\frac{4\sqrt{2}LB}{T\sqrt{\eta_T}}\log(s+T)-\frac{2}{T\sqrt{\eta_T}}\sum_{t=1}^T\frac{\tilde{M}_{t+1}}{\sqrt{\eta_t}}\\
    &\le\frac{8LG}{T^{2/3}}+\frac{8LB}{T^{2/3}}\log(s+T)-
    \frac{2}{T^{2/3}}\sum_{t=1}^T
     \frac{\tilde{M}_{t+1}}{\sqrt{\eta_t}}
    ,
\end{align*}
where the last inequality holds when $T\ge s$. And it also holds that
\begin{align*}
    \frac{1}{T}\sum_{t=1}^T\eta_t\EE\|\Mb_t\|_F&\le \frac{G}{T}+\frac{B}{T}\log(s+T)-
    \frac{\sqrt{2}}{4LT}\sum_{t=1}^T
     \frac{\tilde{M}_{t+1}}{\sqrt{\eta_t}}
    .
\end{align*}
Therefore,
\begin{align*}
    \frac{1}{T}\sum_{t=1}^T\EE\|\Mb_t\|_F
    &\le\frac{G}{T\eta_T}+\frac{B}{T\eta_T}\log(s+T)-
    \frac{\sqrt{2}}{4LT\eta_T}\sum_{t=1}^T
    \frac{\tilde{M}_{t+1}}{\sqrt{\eta_t}}\\
    &\le\frac{G(s+T)^{2/3}}{T}+\frac{B(s+T)^{2/3}}{T}\log(s+T)-
    \frac{\sqrt{2}(s+T)^{2/3}}{4LT}\sum_{t=1}^T
    \frac{\tilde{M}_{t+1}}{\sqrt{\eta_t}}\\
    &\le\frac{2G}{T^{1/3}}+\frac{2B}{T^{1/3}}\log(s+T)-
    \frac{\sqrt{2}}{4LT^{1/3}}\sum_{t=1}^T\frac{\tilde{M}_{t+1}}{\sqrt{\eta_t}}.
\end{align*}
Finally, by triangle inequality, we have
\begin{align*}
    \frac{1}{T}\sum_{t=1}^T\EE\|\nabla F(\Xb_t)\|_F &\leq \frac{1}{T}\sum_{t=1}^T \EE\|\nabla F(\Xb_t)-\Mb_t\|_F + \frac{1}{T}\sum_{t=1}^T \EE\|\Mb_t\|_F\\
    &\leq \frac{1}{T}\sum_{t=1}^T \sqrt{\EE\|\nabla F(\Xb_t)-\Mb_t\|_F^2} + \frac{1}{T}\sum_{t=1}^T\EE\|\Mb_t\|_F\\
    &\leq \sqrt{\frac{1}{T}\sum_{t=1}^T \EE\|\nabla F(\Xb_t)-\Mb_t\|_F^2} + \frac{1}{T}\sum_{t=1}^T\EE\|\Mb_t\|_F\\
    &\le\frac{2\sqrt{2LG}}{T^{1/3}} + \frac{2\sqrt{2LB\log(s+T)}}{T^{1/3}}+\frac{2G}{T^{1/3}}+\frac{2B}{T^{1/3}}\log(s+T)-
    \frac{\sqrt{2}}{4LT^{1/3}}\sum_{t=1}^T\frac{\tilde{M}_{t+1}}{\sqrt{\eta_t}}.
\end{align*}
This completes the proof.
\end{proof}


\end{document}